\documentclass[journal]{IEEEtran}
\usepackage[utf8]{inputenc}
\usepackage{amsmath}
\usepackage{empheq}
\usepackage{amsfonts}
\usepackage{bbold}
\usepackage{booktabs}
\usepackage{tikz}
\usepackage{siunitx}
\usepackage{float}
\usepackage[noadjust]{cite}
\usepackage{hyperref}
\usepackage{graphicx}
\graphicspath{ {./images/} }
\usepackage[english]{babel}
\usepackage{pgfplots}
\usepackage{balance}
\usepackage{mathptmx}
\usepackage{amssymb}
\usepackage{amsthm}

\newtheorem{theorem}{Theorem}

\def\BibTeX{{\rm B\kern-.05em{\sc i\kern-.025em b}\kern-.08em
    T\kern-.1667em\lower.7ex\hbox{E}\kern-.125emX}}

\usetikzlibrary{shapes.geometric}

\usepgfplotslibrary{fillbetween}
\pgfplotsset{compat=1.12}
\pgfdeclarelayer{ft}
\pgfdeclarelayer{bg}
\pgfsetlayers{bg,main,ft}

\makeatletter
\newcommand\footnoteref[1]{\protected@xdef\@thefnmark{\ref{#1}}\@footnotemark}
\makeatother

\DeclareMathSymbol{\shortminus}{\mathbin}{AMSa}{"39}
\newcommand*\widefbox[1]{\fbox{\hspace{2em}#1\hspace{2em}}}

\title{An Input-Output Feedback Linearization based Exponentially Stable Controller for Multi-UAV Payload Transport}

\author{Nishanth Rao, Suresh Sundaram, \IEEEmembership{Senior Member, IEEE}
\thanks{Nishanth Rao is with the Artificial Intelligence and Robotics Lab, Department of Aerospace Engineering, Indian Institute of Science, Bengaluru, India, 560012. \href{nishanthrao@iisc.ac.in}{\texttt{nishanthrao@iisc.ac.in}}}
\thanks{Suresh Sundaram is with the Artificial Intelligence and Robotics Lab, Department of Aerospace Engineering, Indian Institute of Science, Bengaluru, India, 560012. \href{vssuresh@iisc.ac.in}{\texttt{vssuresh@iisc.ac.in}}}}

\begin{document}

\maketitle

\begin{abstract}
In this paper, an exponentially stable trajectory tracking controller is proposed for multi-UAV payload transport. The multi-UAV payload system has a 2-DOF magnetic spherical joint between the UAVs and the vertical rigid links of the payload frame, so the UAVs can roll or pitch freely. These vertical links are rigidly attached to the payload and cannot move. An input-output feedback linearized model is derived for the complete payload-UAV system along with thrust vectoring control for trajectory tracking of the payload. The theoretical analysis on tracking control laws shows that control law is exponentially stable, thus guaranteeing safe transportation along the desired trajectory. To validate the performance of the proposed control law, the results for a numerical simulation as well as a high-fidelity \texttt{Gazebo} real-time simulation are presented. Next, the robustness of the proposed controller is analysed against two practical situations: External disturbance on the payload and payload mass uncertainty. The results clearly indicate that the proposed controller is robust and computationally efficient while achieving exponentially stable trajectory tracking.
\end{abstract}


\section{Introduction}
Technological furtherance in Unmanned Aerial Vehicle (UAV) research has led to advancements in urban air mobility, air logistics, and air ambulances \cite{bertram2019online}, \cite{bauranov2021designing}. A fundamental requirement in these operations is the autonomous transportation of safety-critical payloads (like medical equipment, etc.) from one place to another based on the demand. In particular, due to the versatility of payload characteristics in air-logistics operations, utilising multiple UAVs is often more efficient and economical than having individual UAVs of varying form factors and payload capacity. Consequently, it is crucial to develop algorithms for collaborative transportation of payloads using a swarm of UAVs. Note that these algorithms must be able to govern \textit{different phases} that occur during the transportation process, like individual UAVs moving towards the payload in a particular formation, attaching themselves to the payload automatically, transporting the payload to the desired location and finally detach and fly away to carry on with their individual assigned tasks.

Recently in the literature, researchers assume that the UAVs are already attached to a rigid payload and then developed a model-based tracking controller for cooperative payload transportation. These works can be broadly classified into two categories based on the link suspension employed to connect the UAVs to the payload: Cable suspension or rigid-link suspension. In \cite{sreenath2013dynamics}, a cable-suspended rigid body is considered, and a controller is designed based on the \textit{differential flatness} of the system for the purpose of trajectory tracking. In \cite{lee2017geometric}, a geometric controller is derived for tracking the payload's trajectory using cable suspensions. In \cite{rastgoftar2018cooperative}, the payload is carried by a team of UAVs that are connected via cable-like lightweight \textit{tensegrity structures}, that can handle tension and compressive forces. The feedback linearization technique is applied to only UAVs for position control and assumes the forces from the payload via \textit{external force model}. A PD controller is developed for yaw tracking of the UAVs. However, using cable suspension for payload transportation can be problematic because it can be challenging to control the payload, especially when it swings. Moreover, the problem exacerbates when the cable becomes slack, causing singularities in the mathematical model, as described in \cite{sreenath2013dynamics}. Further, obtaining the orientation of the cable and thus the force direction is a difficult estimation problem in practice. On the contrary, using rigid-link suspensions can alleviate these problems. Moreover, rigid links provide a better control authority over the payload to the UAVs and thus, stabilizing the payload is easier. In \cite{wehbeh2020distributed}, a distributed Model Predictive Controller (MPC) is designed for trajectory tracking of the payload that is suspended using rigid links. It is shown that the distributed MPC performs similar to the centralized MPC only if the controller frequency is increased. In \cite{rao2022integrated}, the \textit{exponential barrier functions} are employed along with a centralised linear MPC for dynamic obstacle avoidance in a cluttered environment.

In \cite{mellinger2013cooperative}, different phases that arise in payload transportation are described where individual UAVs come together to grasp a payload using a \textit{vacuum gripping mechanism} in various flight configurations. Here, the form factor of the payload is similar to, if not smaller than, the UAVs themselves. Nonetheless, when there is a bigger and heavier payload to be transported, the gripping mechanism can have many disadvantages, particularly the loss of thrust and the limited attitude freedom of the UAVs, as they attach too close to the payload surface. For the UAVs to freely roll/pitch and to ensure there is no loss of thrust, the UAVs must attach themselves at a certain height above the payload surface.

When there are critical payloads that must be transported in emergency situations, the controller must track the reference trajectory precisely, which must be guaranteed theoretically. In addition, one must also consider developing a \textit{complete control} approach that governs the different phases during the transportation process, like governing the behaviour of independent UAVs while coming into formation above the payload for precise attachment process and the transportation of the payload to its destination. Further, the controller tracking performance must be \textit{exponentially stable} to guarantee the precise transportation of safety-critical payload along a set trajectory.

In this paper, a novel feedback-linearization based nonlinear control algorithm for both the individual UAV trajectory tracking and cooperative safety-critical payload transport is proposed to address different phases that arise in a typical payload transportation problem. The UAVs are connected to the payload frame in cooperative operation through a vertical rigid link suspension system. It is assumed that there is a 2-DOF magnetic spherical joint between the UAV and the vertical rigid links. This allows the UAVs to roll/pitch freely. The vertical rigid links ensure that the UAVs attach at a height above the payload, mitigating the issues of thrust loss. The control structure employs feedback-linearization for individual UAV trajectory tracking and multiple-UAV cooperative payload trajectory tracking and transport. For the UAVs, the \textit{exact-feedback linearization} procedure as described by \cite{mistler2001exact} is employed. In the case of a multi-UAV payload system, the system dynamics are derived firstly using \textit{Lagrangian Mechanics}. Then, using the \textit{input-output feedback-linearization} process, a \textit{feedback-linearized} model is presented. Based on these feedback-linearized models, a \textit{globally exponentially stable} tracking control law is derived, and the exponential stability is established theoretically. For the payload-UAV system, the {thrust-vectoring} control approach \cite{silva2018quadrotor} is used to obtain actual system inputs from the feedback-linearized control inputs.  Numerical simulation results are presented to analyse the tracking performance of both the controllers. The control algorithm is also evaluated to test its robustness against external disturbance (i.e., wind influence on the payload) and payload mass uncertainty. Later, a high-fidelity software-in-loop simulation is conducted in \texttt{Gazebo} simulator to verify the real-time performance of the proposed control law and validate its computational efficiency.

The rest of the paper is organized as follows: In Section \ref{sec:system_model}, the \textit{extended UAV dynamics} \cite{mistler2001exact} is briefly discussed. Then the equations of motion for the payload-UAV system are derived using \textit{Lagrangian Mechanics} and a state-space representation is obtained. In Section \ref{sec:input-output_fbl}, the feedback linearization procedure is presented for the payload-UAV system, and the tracking laws are derived. In Section \ref{sec:results}, the proposed control laws are evaluated and the results are presented for both the numerical as well as \texttt{Gazebo} simulations. The robustness of the control algorithm is demonstrated by testing it in two cases that often arise in practice: external disturbance on the payload, and payload mass uncertainty. Section \ref{sec:future} concludes the work and discusses the possibilities for future directions.

\section{Preliminaries and System Model}
\label{sec:system_model}

In this section, the preliminary on input-output feedback linearization, the extended UAV dynamics and the exact feedback linearization of individual UAV is discussed. Next the equations of motion for the multi-UAV payload system is derived using Lagrangian mechanics. Finally, a state-space representation of the multi-UAV payload system is provided.

\subsection{Input-Output Feedback Linearization}
The goal of input-output feedback linearization is to algebraically transform the output $y = h(x), y \in \mathbb{R}^p$ of a nonlinear system characterized by the dynamics $\dot{x} = f(x) + g(x)u, x\in\mathbb{R}^n, u\in\mathbb{R}^m$ into the output $y^{(r)} = v, v \in \mathbb{R}^p$ characterized by the linear system $\dot{z} = Az + Bv$, where $z$ is the new coordinate system, $r\in \mathbb{R}^p$ is the \textit{vector relative degree} and $v$ is the \textit{new transformed input} to the system.

In particular, if we can re-write the output as $y^{(r)} = b(x) + \Delta(x)u$ with $\Delta(x)$ invertible, then a state-feedback control law of the form $u = -\Delta(x)^{-1}b(x) + \Delta(x)^{-1}v$ can be found that renders the closed loop system linear and decoupled, with the new output being $y_i^{(r_i)} = v_i, i=1,..,p$. The notation $y_i^{(r_i)}$ denotes the $r_i^{\textit{th}}$ derivative of $y_i$. The quantities $\Delta(x), b(x)$ can be computed as:
\begin{align}
    \Delta(x) & = \left[
        \begin{array}{ccc}
             \mathfrak{L}_{g_1}\mathfrak{L}_f^{r_1 - 1}h_1(x) & \hdots & \mathfrak{L}_{g_m}\mathfrak{L}_f^{r_1 - 1}h_1(x) \\
             \vdots & \ddots & \vdots \\
             \mathfrak{L}_{g_m}\mathfrak{L}_f^{r_p - 1}h_p(x) & \hdots & \mathfrak{L}_{g_m}\mathfrak{L}_f^{r_p - 1}h_p(x)
        \end{array}
    \right] \label{eq:delta_x} \\ & \text{and} \ \ \ \ \  b(x) = \left[ 
        \begin{array}{c}
             \mathfrak{L}_f^{r_1}h_1(x) \\
             \vdots \\
             \mathfrak{L}_f^{r_p}h_p(x)
        \end{array}
    \right] \label{eq:b_x}
\end{align}
where $\mathfrak{L}_ab(x) = \frac{\partial b}{\partial x} \cdot a(x)$ denotes the \textit{lie derivative} of the vector field $b(x)$ along the vector field $a(x)$.

\begin{figure}
    \centering
    
    \begin{tikzpicture}[scale=1.2, every node/.style={transform shape}]
        
        \draw (0,0) rectangle (2.2, 1);
        \node at (1.1, 0.5) {\tiny$-\Delta^{-1}(\bar{X})b(\bar{X}) + \Delta^{-1}(\bar{X})v$ };
        
        \draw[-latex] (-0.3, 0.2) -- (0.0, 0.2);
        \draw[-latex] (-0.3, 0.4) -- (0.0, 0.4);
        \draw[-latex] (-0.3, 0.6) -- (0.0, 0.6);
        \draw[-latex] (-0.3, 0.8) -- (0.0, 0.8);
        
        \node at (-0.45, 0.2) {\tiny$v_4$};
        \node at (-0.45, 0.4) {\tiny$v_3$};
        \node at (-0.45, 0.6) {\tiny$v_2$};
        \node at (-0.45, 0.8) {\tiny$v_1$};
        
        \draw[-latex] (2.2, 0.8) -- (2.6, 0.8);
        \draw (2.6, 0.7) rectangle (2.8, 0.9);
        \node at (2.7, 0.8) {\tiny$\pmb{\int}$};
        \draw[-latex] (2.8, 0.8) -- (3.2, 0.8);
        \draw (3.2, 0.7) rectangle (3.4, 0.9);
        \node at (3.3, 0.8) {\tiny$\pmb{\int}$};
        \draw[-latex] (3.4, 0.8) -- (4.0, 0.8);
        
        \draw[-latex] (2.2, 0.6) -- (4.0, 0.6);
        \draw[-latex] (2.2, 0.4) -- (4.0, 0.4);
        \draw[-latex] (2.2, 0.2) -- (4.0, 0.2);
        
        \node at (2.35, 0.7) {\tiny$\bar{u}_1$};
        \node at (2.35, 0.5) {\tiny$\bar{u}_2$};
        \node at (2.35, 0.3) {\tiny$\bar{u}_3$};
        \node at (2.35, 0.1) {\tiny$\bar{u}_4$};
        
        \node at (3.8, 0.7) {\tiny$u_1$};
        \node at (3.8, 0.5) {\tiny$u_2$};
        \node at (3.8, 0.3) {\tiny$u_3$};
        \node at (3.8, 0.1) {\tiny$u_4$};
        
        \node at (2.9, 0.7) {\tiny$\xi$};
        \node at (3.5, 0.7) {\tiny$\zeta$};
        
        \draw (4.0, 0.0) rectangle (6.0, 1.0);
        \node at (5.0, 0.5) {\scriptsize UAV Dynamics};
        
        \draw[-latex] (5.0, 0.0) -- (5.0, -0.5) -- (1.1, -0.5) -- (1.1, 0.0);
        
        \node at (2.9, -0.35) {\scriptsize UAV State};
        \draw[-latex] (2.9, 0.8) -- (2.9, 1.3) -- (1.1, 1.3) -- (1.1, 1.0);
        \draw[-latex] (3.6, 0.8) -- (3.6, 1.7) -- (0.7, 1.7) -- (0.7, 1.0);
        
    \end{tikzpicture}
    
    \caption{Schematic diagram of the feedback linearization for UAV system}
    \label{fig:sch_diag_fl_UAV}
\end{figure}
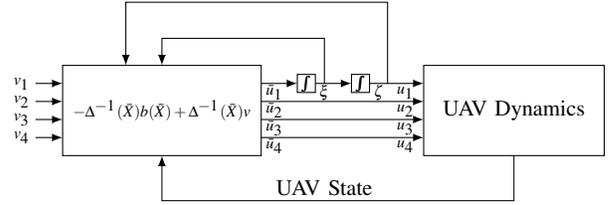 

\subsection{An Extension to UAV dynamics}
The standard UAV dynamics can be written as:
\begin{align}
    \dot{\chi} = f(\chi) + \sum_{i=1}^4 g_i(\chi)U_i \label{eq:uav}
\end{align}
with
\begin{align}
    \chi &= \left[ r_x \ r_y \ r_z \ \psi \ \theta \ \phi \ \dot{r}_x \ \dot{r}_y \ \dot{r}_z \ p \ q \ r \right]^T \\
    f(\chi) &= \left[\begin{array}{c} \dot{r}_x \\ \dot{r}_y \\ \dot{r}_z \\ qsin(\phi)sec(\theta) + rcos(\phi)sec(\theta) \\ qcos(\phi) - rsin(\phi) \\ p + qsin(\phi)tan(\theta) + rcos(\phi)tan(\theta) \\ 0 \\ 0 \\ g \\ \frac{I_y - I_z}{Ix}qr \\ \frac{I_z - I_x}{I_y}pr \\ \frac{I_x - I_y}{I_z}pq \end{array}\right] \\
    g_1(\chi) &= \left[ 0 \ 0 \ 0 \ 0 \ 0 \ 0 \ o_1 \ o_2 \ o_3 \ 0 \ 0 \ 0 \right]^T \\
    g_2(\chi) &= \left[0 \ 0 \ 0 \ 0 \ 0 \ 0 \ 0 \ 0 \ 0 \ I_x^{-1} \ 0 \ 0 \right]^T \\
    g_3(\chi) &= \left[0 \ 0 \ 0 \ 0 \ 0 \ 0 \ 0 \ 0 \ 0 \ 0 \ I_y^{-1} \ 0 \right]^T \\
    g_4(\chi) &= \left[0 \ 0 \ 0 \ 0 \ 0 \ 0 \ 0 \ 0 \ 0 \ 0 \ 0 \ I_z^{-1} \right] \\
    \text{and}, & \\
    o_1 &= \frac{-1}{m}\left(cos(\phi)cos(\psi)sin(\theta) + sin(\phi)sin(\psi) \right) \\
    o_2 &= \frac{-1}{m}\left(cos(\phi)sin(\theta)sin(\psi) - cos(\psi)sin(\phi) \right) \\
    o_3 &= \frac{-1}{m}\left(cos(\theta)cos(\phi) \right)
\end{align}
where $(r_x, r_y, r_z)$ denote the $x-, y-, z-$ position of the UAV, $(\psi, \theta, \phi)$ denote the Euler angles namely yaw, pitch and roll angles of the UAV, $(\dot{r}_x, \dot{r}_y, \dot{r}_z)$ are the $x-, y-, z-$ linear velocities, $(p, q, r)$ are the angular velocities about the $x-, y-, z-$ axis and $U = \left[U_1 \ U_2 \ U_3 \ U_4 \right] = \left[f_t \ \tau_x \ \tau_y \ \tau_z \right]$ are the thrust and torque control inputs.

In order to carry out the input-output feedback linearization of the UAV system as described by Eq. \ref{eq:uav}, one has to obtain the matrix $\Delta(\chi)$ from Eq. \ref{eq:delta_x}. Unfortunately, for the standard UAV dynamics, the matrix $\Delta$ is non-invertible for any state $\chi$. As a workaround, the \textit{extended UAV dynamics} as presented in \cite{mistler2001exact} is considered, for which the matrix $\Delta$ is invertible in a certain state-space region.

In the extended UAV dynamics, the thrust control input is delayed by a double integrator, and the remaining three torque controls remain the same.
\begin{align}
    \text{Let} \ U &= \left[f_t \ \tau_x \ \tau_y \ \tau_z \right] = \left[U_1 \ U_2 \ U_3 \ U_4 \right] \in \mathbb{R}^4 \\
    \text{Define} \ U_1 &= \zeta, \ \dot{\zeta} = \xi, \ \text{and,} \ \dot{\xi} = \bar{U}_1 \ \\
    \text{Thus}, \ \bar{U} & = \left[ \bar{U}_1 \ \bar{U_2} \ \bar{U_3} \ \bar{U_4} \right] = \left[ \bar{U}_1 \ U_2 \ U_3 \ U_4 \right] \in \mathbb{R}^4 \label{eq:ext_u}
\end{align}

The quantities $\zeta, \xi$ become the internal states of the UAV, and thus the \textit{extended} state vector of the UAV $\bar{X} \in \mathbb{R}^{14}$ is redefined as:
\begin{align}
    \bar{X} = \left[ r_x \ r_y \ r_z \ \psi \ \theta \ \phi \ \dot{r}_x \ \dot{r}_y \ \dot{r}_z \ \zeta \ \xi \ p \ q \ r \right]^T \label{eq:ext_uav_state}
\end{align}
 Then, the \textit{extended} UAV dynamics is as follows:
\begin{align}
    \dot{\bar{X}} = \bar{f}(\bar{X}) + \sum_{i=1}^4 \bar{g}_i(\bar{X})\bar{U}_i
    \label{eq:ext_uav_model}
\end{align}
where the functions $\bar{f}$ and $\bar{g}_i$ are given by:

\begin{align*}
    &\bar{f}(\bar{X}) = \left[\begin{array}{c} \dot{r}_x \\ \dot{r}_y \\ \dot{r}_z \\ q \ sin(\phi)sec(\theta) + r \ cos(\phi)sec(\theta) \\ q \ cos(\phi) - r \ sin(\phi) \\ p+qsin(\phi)tan(\theta) + rcos(\phi)tan(\theta) \\ \frac{-\zeta}{m}\left( cos(\phi)cos(\psi)sin(\theta) + sin(\phi)sin(\psi) \right) \\ \frac{-\zeta}{m}\left( cos(\phi)sin(\theta)sin(\psi) - cos(\psi)sin(\phi) \right) \\ \frac{-\zeta}{m}\left( cos(\theta)cos(\phi) \right) \\ \xi \\ 0 \\ \frac{I_y - I_z}{Ix}qr \\ \frac{I_z - I_x}{I_y}pr \\ \frac{I_x - I_y}{I_z}pq \end{array} \right] \\
    &\text{and,} \\ 
    &\bar{g}_1(\bar{X}) = \left[0 \ 0 \ 0 \ 0 \ 0 \ 0 \ 0 \ 0 \ 0 \ 0 \ 1 \ 0 \ 0 \ 0 \right]^T \\
    &\bar{g}_2(\bar{X}) = \left[0 \ 0 \ 0 \ 0 \ 0 \ 0 \ 0 \ 0 \ 0 \ 0 \ 0 \ I_x^{-1} \ 0 \ 0 \right]^T \\
    &\bar{g}_3(\bar{X}) = \left[0 \ 0 \ 0 \ 0 \ 0 \ 0 \ 0 \ 0 \ 0 \ 0 \ 0 \ 0 \ I_y^{-1} \ 0 \right]^T \\
    &\bar{g}_4(\bar{X}) = \left[0 \ 0 \ 0 \ 0 \ 0 \ 0 \ 0 \ 0 \ 0 \ 0 \ 0 \ 0 \ 0 \ I_z^{-1} \right]^T
\end{align*}

\subsection{Exact feedback linearization of Extended UAV dynamics}
For the UAV system, the output is chosen as $y = h(x) = \left[r_x \ r_y \ r_z \ \psi \right]^T$. As discussed previously, the term $\Delta(\chi)$ is non-invertible for any state $\chi$ for the standard UAV dynamics. However, the extended UAV dynamics renders $\Delta(\bar{X})$ invertible $ \forall \ \ \zeta \neq 0, -\frac{\pi}{2} < \phi, \theta < \frac{\pi}{2}$\footnote{The matrix $\Delta(\bar{X})$ and $b(\bar{X})$ can be found in the supplementary material \href{https://drive.google.com/drive/folders/1ASU1YWjOVZp5kzjpX9h4WE2W88JTJczX?usp=sharing}{(here)} or in \cite{nish2022supp_mat} for reference.}. The vector relative degree is $r = \left[4 \ 4 \ 4 \ 2 \right]$. In other words, the $x-, y-, z-$ position of the UAV must be differentiated 4 times and the yaw angle must be differentiated 2 times, so that the control input $\bar{u}$ appears explicitly in the form of $b(\bar{X}) + \Delta(\bar{X})\bar{U}$, with $b(\bar{X}) \in \mathbb{R}^4$ and $\Delta(\bar{X}) \in \mathbb{R}^{4\times4}$.
Additionally, since the extended system in Eq. (\ref{eq:ext_uav_state}) has 14 states which also equals the total sum of the relative degrees of the output, the transformed system $y^{(r)} = v$ can be written in a \textit{fully linear and controllable form}\cite{isidori1985nonlinear}\cite{nijmeijer1990nonlinear}. Thus, the transformed output of the UAV is given by:
\begin{empheq}[box=\widefbox]{align}
    y^{(r)} = \left[\begin{array}{c}
         \ddddot{r}_x \\
         \ddddot{r}_y \\
         \ddddot{r}_z \\
         \ddot{\psi} 
    \end{array}\right] = 
    \left[\begin{array}{c}
         v_1 \\ v_2 \\ v_3 \\ v_4
    \end{array}\right] = v \label{eq:fl_uav_dyn}
\end{empheq}
The feedback linearization technique has been applied to only UAV dynamics in \cite{rastgoftar2018cooperative} and \cite{aghdam2016cooperative}, and the payload dynamics is modelled as an external force to the UAV. In this paper, the feedback linearization is developed for the entire payload-UAV system in Section \ref{sec:input-output_fbl}.

\subsection{Payload-UAV System Description}

In general, consider $N$ UAVs that are attached to the vertical massless rigid links of lengths $l_i$ via 2-DOF spherical joints. The standard NED (North, East, Down) right-handed coordinate system is used for payload-UAV system modelling, with the positive $z$-axis pointing downwards.
The location of the center of mass of the payload with a mass $m_0 \in \mathbb{R}$ in the inertial frame is denoted by $r_0 \in \mathbb{R}^3$, and the center of mass of the $i^{th}$ UAV with a mass $m_i \in \mathbb{R}$ in the inertial frame is denoted by $r_i \in \mathbb{R}^3$. The attachment point of the $i^{th}$ rigid link is denoted by $\rho_i \in \mathbb{R}^3$ in the \textit{payload-fixed frame}, where the origin coincides with the payload's center of mass. The unit vector expressed in the payload-fixed frame whose direction is along each of the rigid link at any given time is denoted by $\mathbb{k} = \left[ 0 \ 0 \ 1 \right]^T \in \mathbb{R}^3$. The attitude of the payload and the $i^{th}$ UAV is characterized by $\pmb{R}_0 \in SO(3)$ and $\pmb{R}_i \in SO(3)$, which is the rotation matrix that rotates a vector in the body-fixed frame to the inertial frame. The inertia matrix of the payload and the UAV is denoted by $\pmb{J}_0 \in \mathbb{R}^{3\times3}$ and $\pmb{J}_i \in \mathbb{R}^{3\times3}$ respectively.

The force vector generated by the $i^{th}$ UAV in the inertial frame is denoted by $f_i \in \mathbb{R}^3$, which is related to the overall thrust force $f_{t_i} \in \mathbb{R}$ produced by the motors of the $i^{th}$ UAV as $f_i = f_{t_i}\pmb{R}_i\mathbb{k}$. The motors also produce a torque vector $\tau_i = \left[ \tau_{x_i} \ \tau_{y_i} \ \tau_{z_i} \right]^T \in \mathbb{R}^3$ which is expressed in the body-fixed frame of the UAV. The overall control input to the payload-UAV system can then be chosen as $\left\{ f_{t_i}, \tau_i\right\}$ with $i\in 1,\hdots,N$.

\subsection{Payload-UAV Dynamics and State-space Representation}
With the notations and description of the payload-UAV system previously discussed, the equations of motion for the payload-UAV system are derived. As the rigid links are always vertical, the position of the $i^{th}$ UAV can be directly calculated as:
\begin{align}
    r_i = r_0 + \pmb{R}_0\left(\rho_i - l_i\mathbb{k} \right)
\end{align}
The kinematic equations are given by:
\begin{align}
    \dot{r}_0 & = \pmb{R}_0v_0 \label{eq:kin_beg}\\
    \dot{\pmb{R}}_0 & = \pmb{R}_0\omega_0^\times \\
    \dot{r}_i & = \dot{r}_0 + \pmb{R}_0\omega_0^\times \left(\rho_i - l_i\mathbb{k} \right) \\
    \dot{\pmb{R}}_i & = \pmb{R}_i\omega_i^\times \label{eq:kin_end}
\end{align}
where the skew-symmetric operator $(.)^\times : \mathbb{R}^3 \rightarrow SO(3)$ denotes the hat map, $v_0 \in \mathbb{R}^3$ denotes the linear velocity of the payload in the payload-fixed frame, $\omega_0\in\mathbb{R}^3$ and $\omega_i\in\mathbb{R}^3$ denote the angular velocities of the payload and the $i^{th}$ UAV respectively.

\begin{figure}
    \begin{tikzpicture}
    
        \node[name path=A, trapezium, draw, minimum width=3.5cm, trapezium left angle=120, trapezium right angle=60, line width=0mm, fill=brown!15] at (2.5,-3) {};
        \node(Y)[name path=B, trapezium, draw, minimum width=4cm, trapezium left angle=120, trapezium right angle=60, line width=0mm] at (2.5,-3) {};
        
        \draw (3.3, -2.7) -- (3.3, -2.2);
        \draw (1.05, -2.7) -- (3.3, -2.7) -- (3.95, -3.8);
        
        \draw[name path=C] (0.5, -2.08) -- (1.57, -3.93) -- (4.5, -3.93);
        \draw[name path=D] (0.5, -2.8) -- (1.57, -4.5) -- (4.5, -4.5);
        
        \draw (0.5, -2.08) -- (0.5, -2.8);
        \draw (1.57, -4.5) -- (1.57, -3.93);
        \draw (4.5, -4.5) -- (4.5, -3.93);
        
        \tikzfillbetween[of=A and B, on layer=ft]{brown, opacity=0.5};
        \tikzfillbetween[of=C and D, on layer=ft]{brown, opacity=0.5};
        
        \draw[name path=E] (0.55, -2.9) -- (0.55, 1.0); 
        \draw[name path=F] (0.6, -3.0) -- (0.6, 0.95);
        \tikzfillbetween[of=E and F, on layer=ft]{black, opacity=0.7};
        
        \draw[name path=G] (1.5, -4.4) -- (1.5, -1.0);
        \draw[name path=H] (1.45, -4.3) -- (1.45, -0.95);
        \tikzfillbetween[of=G and H, on layer=ft]{black, opacity=0.7};
        
        \draw[name path=I] (4.5, -3.95) -- (4.5, -1.0);
        \draw[name path=J] (4.45, -3.85) -- (4.45, -0.95);
        \tikzfillbetween[of=I and J, on layer=ft]{black, opacity=0.7};
    
        \draw[name path=K] (3.55, -2.3) -- (3.55,0.95);
        \draw[name path=L] (3.5, -2.2) -- (3.5, 1.0);
        \tikzfillbetween[of=K and L, on layer=ft]{black, opacity=0.7};
        
        \node[cylinder, draw=black!40, shape border rotate=90, minimum width=1.7cm, minimum height=0.8cm, cylinder uses custom fill, cylinder body fill = black!40, cylinder end fill = black!60, opacity=0.4] (c) at (2.5,-3.3) {};
        
        \node[ellipse, draw, fill = gray!60, minimum width = 0.55cm, minimum height = 0.08cm] (prop11) at (0.1,1.15 - 0.45) {};
        \node[ellipse, draw, fill = gray!60, minimum width = 0.55cm, minimum height = 0.08cm] (prop12) at (1.0,1.15 - 0.45) {};
        \node[ellipse, draw, fill = gray!60, minimum width = 0.55cm, minimum height = 0.08cm] (prop13) at (1.0,1.7 - 0.45) {};
        \node[ellipse, draw, fill = gray!60, minimum width = 0.55cm, minimum height = 0.08cm] (prop14) at (0.1,1.7 - 0.45) {};
        \draw[line width=0mm] (prop11) -- (prop13);
        \draw[line width=0mm] (prop12) -- (prop14);
        
        \node[ellipse, draw, fill = gray!60, minimum width = 0.55cm, minimum height = 0.08cm] (prop21) at (1.05,-0.8 - 0.45) {};
        \node[ellipse, draw, fill = gray!60, minimum width = 0.55cm, minimum height = 0.08cm] (prop22) at (1.95,-0.8 - 0.45) {};
        \node[ellipse, draw, fill = gray!60, minimum width = 0.55cm, minimum height = 0.08cm] (prop23) at (1.8,-0.25 - 0.45) {};
        \node[ellipse, draw, fill = gray!60, minimum width = 0.55cm, minimum height = 0.08cm] (prop24) at (1.0,-0.25 - 0.45) {};
        \draw[line width=0mm] (prop21) -- (prop23);
        \draw[line width=0mm] (prop22) -- (prop24);
        
        \node[ellipse, draw, fill = gray!60, minimum width = 0.55cm, minimum height = 0.08cm] (prop31) at (1.05 + 3,-0.8 - 0.45) {};
        \node[ellipse, draw, fill = gray!60, minimum width = 0.55cm, minimum height = 0.08cm] (prop32) at (1.95 + 3,-0.8 - 0.45) {};
        \node[ellipse, draw, fill = gray!60, minimum width = 0.55cm, minimum height = 0.08cm] (prop33) at (1.8 + 3,-0.25 - 0.45) {};
        \node[ellipse, draw, fill = gray!60, minimum width = 0.55cm, minimum height = 0.08cm] (prop34) at (1.0 + 3,-0.25 - 0.45) {};
        \draw[line width=0mm] (prop31) -- (prop33);
        \draw[line width=0mm] (prop32) -- (prop34);
        
        \node[ellipse, draw, fill = gray!60, minimum width = 0.55cm, minimum height = 0.08cm] (prop11) at (0.1 + 3,1.15 - 0.45) {};
        \node[ellipse, draw, fill = gray!60, minimum width = 0.55cm, minimum height = 0.08cm] (prop12) at (1.0 + 3,1.15 - 0.45) {};
        \node[ellipse, draw, fill = gray!60, minimum width = 0.55cm, minimum height = 0.08cm] (prop13) at (1.0 + 3,1.7 - 0.45) {};
        \node[ellipse, draw, fill = gray!60, minimum width = 0.55cm, minimum height = 0.08cm] (prop14) at (0.1 + 3,1.7 - 0.45) {};
        \draw[line width=0mm] (prop11) -- (prop13);
        \draw[line width=0mm] (prop12) -- (prop14);
        
        

		\node at (2.5, -5.0) {\large $r_0, \ \pmb{R}_0$};
		
        \node at (3.5, 1.8) {\large $r_i, \ \pmb{R}_i$};
        
		\draw[-latex] (2.65, -3.3) -- (3.32, -2.7);
		\node at (2.65, -3.3)[circle,fill,inner sep=1.5pt]{};        
		\node at (2.6, -3.0) {\large $\rho_i$};
        
        \draw[<-] (-0.7, -5.2) -- (-1.2, -5.0);
        \draw[->] (-1.2, -5.0) -- (-1.2, -5.6);
        \draw[->] (-1.2, -5.0) -- (-0.8, -4.6);
        
        \node at (-0.6, -4.6) {\tiny x};
        \node at (-0.6, -5.2) {\tiny y};
        \node at (-1.3, -5.7) {\tiny z};
        
        \draw[-latex] (3.55, 0.95) -- (4.0, -0.1);
        \node at (4.2, 0.1) {$f_i$};
        
        \draw[{|}{latex}-{latex}{|}] (0.45, 0.95) -- (0.45, -2.1);
        \node at (0.3, -0.7) {$l_i$};
        
    \end{tikzpicture}
    
    \caption{A schematic diagram illustrating the payload transport using multiple UAVs. The support-frame together with the payload (shown lightly as a black cylinder) is modelled as a cuboid.}
    \label{fig:schematic_diagram_payloadUAV}
    
\end{figure}
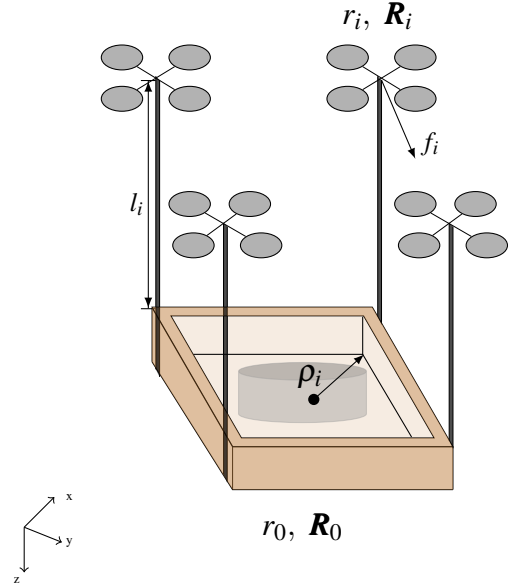

\textit{Lagrangian formulation} is used to describe the dynamics of the payload-UAV system. The kinetic energy $\mathcal{T}$ and the potential energy $\mathcal{U}$ of the system can be obtained as:
\begin{align}
    \begin{split}
        \mathcal{T} = {}& \frac{1}{2} m_0 \Vert \Dot{r}_0 \Vert^2_{_2} + \frac{1}{2} \omega_0 \cdot \pmb{J}_0\omega_0 \\ 
        & + \sum_{i=1}^N \left( \frac{1}{2} m_i \Vert \Dot{r}_i \Vert^2_{_2} + \frac{1}{2} \omega_i \cdot \pmb{J}_i \omega_i \right)
    \end{split} \\
    \begin{split}
        \mathcal{U} = {}& -m_0 g \mathbb{k} \cdot r_0 - \sum_{i=1}^N m_ig\mathbb{k} \cdot r_i
    \end{split} \\
    \begin{split}
        \mathcal{L} = {}& \mathcal{T} - \mathcal{U}
    \end{split}
\end{align}
where $\mathcal{L}$ is the Lagrangian of the payload-UAV dynamics. As described in \cite{lee2017geometric}, the \textit{Lagrange-d'Alembert Principle} is used to obtain the following equations:
\begin{align}
    \frac{d}{dt} \pmb{D}_{\dot{r_0}}\mathcal{L} - \pmb{D}_{r_0}\mathcal{L} & = \sum_{i=1}^N f_i \\
    \frac{d}{dt} \pmb{D}_{\omega_0} + \omega_0^\times\pmb{D}_{\omega_0} - \pmb{d}_{\pmb{R}_0}\mathcal{L} & = \sum_{i=1}^N \rho_i^\times \pmb{R}_0^Tf_i \\
    \frac{d}{dt} \pmb{D}_{\omega_i} \mathcal{L} + \omega_i^\times \pmb{D}_{\omega_i}\mathcal{L} & = \tau_i
\end{align}
where $\pmb{D}_a \mathcal{L}$ is the derivative of the Lagrangian $\mathcal{L}$ with respect to the vector $a$, and the quantity $\pmb{d}_{\pmb{R}_0}$ are known as the \textit{left-trivialized derivatives}\cite{lee2017geometric}. Simplifying the above set of equations yields the dynamics of the payload-UAV system:
\begin{gather}
    \begin{split}
        m_T\left(\dot{v}_0 + \omega_0^\times v_0\right) + \sum_{i=1}^N m_i \left( -\rho_i^\times\dot{\omega}_0 + \left(\omega_0^\times \right)^2\rho_i \right) \\ = \pmb{R}_0^T\left( m_Tg\mathbb{k} + \sum_{i=1}^N f_i \right) \ \ \ \ \ \ \label{eq:v_0}
    \end{split} \\
    \begin{split}
        \sum_{i=1}^Nm_i\rho_i^\times\left( \dot{v}_0 + \omega_0^\times v_0 \right) + \pmb{\Bar{J}}_0\dot{\omega}_0 + \omega_0^\times\pmb{\Bar{J}}_0\omega_0 \\ = \sum_{i=1}^N\rho_i^\times\pmb{R}_0^T\left( f_i + m_ig\mathbb{k} \right) \ \ \ \ \ \label{eq:w_0}
    \end{split} \\
    \begin{split}
        \pmb{J}_i\dot{\omega}_i + \omega_i^\times \pmb{J}\omega_i = \tau_i \label{eq:w_i}
    \end{split}
\end{gather}
where the quantity $m_T = m_0 + \sum_{i=1}^Nm_i$ is the total combined mass, and $\pmb{\bar{J}}_0 = \left( \pmb{J}_0 - \sum_{i=1}^N m_i\left( \rho_i^\times \right)^2 \right)$ is the \textit{apparent} moment of inertia of the payload.

In order to obtain a state-space representation of the payload-UAV dynamic model, define the state vector as:
\begin{align}
    X = \left[\underbrace{r_0^T \ v_0^T \ \Theta_0^T \ \omega_0^T}_\text{payload $\in \mathbb{R}^{12}$} \underbrace{\Theta_i^T \ \omega_i^T}_\text{$i^{th}$ UAV $\in \mathbb{R}^{6}$} \right]^T, \ i \in \{1,..., N \}
    \label{eq:state_def}
\end{align}
where $\Theta_0, \Theta_i\in\mathbb{R}^3$ are the $ZYX$ Euler angle characterization of the attitudes $\pmb{R}_0, \pmb{R}_i$ respectively. For the set of states $\dot{X}_a = \left[ r_0 \ \Theta_0 \ \Theta_i \ \omega_i \right]^T$, the state-space representation can be easily obtained from the set of Eq. (\ref{eq:kin_beg}) - (\ref{eq:kin_end}) and Eq. (\ref{eq:w_i}). For the remaining states $\dot{X}_b = \left[ v_0 \ \omega_0 \right]^T$, the state-space representation must be obtained by substituting and rearranging Eq. (\ref{eq:v_0}) and Eq. (\ref{eq:w_0}), which can be compactly represented by:
\begin{align}
    \dot{X}_b = \pmb{P}\pmb{Q} \label{eq:dot_x_b}
\end{align}
where $\pmb{P} \in \mathbb{R}^{6\times 6}$ is given by:
\begin{equation}
    \left[ \begin{array}{c|c}
         m_T\pmb{I}_3 & -\sum_{i=1}^N m_i\rho_i^\times \\
         \hline \sum_{i=1}^N m_i\rho_i^\times & \pmb{\bar{J}}_0
    \end{array}\right]^{-1} = \left[ 
    \begin{array}{c|c}
         \pmb{P}_{11} & \pmb{P}_{12} \\
         \hline
         \pmb{P}_{21} & \pmb{P}_{22}
    \end{array} \right] \label{eq:matrix_P}
\end{equation}
and $\pmb{Q} \in \mathbb{R}^{6\times 1}$ is given by:
\begin{equation}
\resizebox{.94\hsize}{!}{$
    \left[ \begin{array}{c}
         -m_{T} \omega_{0}^{\times} v_{0}-\sum_{i=1}^{N}m_{i}\left(\omega_{0}^{\times}\right)^{2} \rho_{i}+m_{T} g \pmb{R}_{0}^{T} \mathbb{k}+\sum_{i=1}^{N} \pmb{R}_{0}^{T} f_{i} \\ 
         \hline 
         -\omega_{0}^{\times} \Bar{\pmb{J}}_{0} \omega_{0}-\sum_{i=1}^{N} m_{i}\rho_{i}^{\times} \omega_{0}^{\times} v_{0}+\sum_{i=1}^{N} \rho_{i}^{\times} \pmb{R}_{0}^{T}\left(f_{i}+m_{i} g \mathbb{k}\right)
    \end{array}\right]$}
\end{equation}
where $\pmb{I}_3$ is the $3\times 3$ Identity matrix. The control input to the system is $\mathbb{U} \in \mathbb{R}^{4N}$:
\begin{align}
    \mathbb{U} = \left[f_{t_i} \ \tau_{x_i} \ \tau_{y_i} \ \tau_{z_i} \right]^T, \ \ i=1,\hdots,N \label{eq:payload_uav_u}
\end{align}
The structure of the resultant $\pmb{P}$ matrix can be found in the supplementary material \cite{nish2022supp_mat}.
\section{Input-Output Feedback Linearization and Exponentially stable tracking control law}
\label{sec:input-output_fbl}

In this section, the input-output feedback linearization of the payload-UAV system dynamics is presented. Next, an \textit{exponentially stable} tracking control law is designed for the UAV system and the payload-UAV system. Finally, the \textit{thrust-vectoring} control approach is used for the payload-UAV system to transform the feedback-linearized inputs to the actual system inputs.

\subsection{Feedback linearization of payload-UAV system}
\label{sec:input-output_fbl_sub_a}
The output for the payload-UAV system is chosen as the position of the payload i.e., $y_0 = h(X) = r_0$, for the purpose of payload-trajectory tracking. From Eq. (\ref{eq:dot_x_b}), it can be seen that the control input appears in the equation for $\dot{v}_0$, which can be written as:
\begin{align}
    \begin{split}
    \dot{v}_0  &= \pmb{P}_{11}\left( -m_{T} \omega_{0}^{\times} v_{0}-\sum_{i=1}^{N}m_{i}\left(\omega_{0}^{\times}\right)^{2} \rho_{i}+m_{T} g \pmb{R}_{0}^{T} \mathbb{k} \right) \\ &+ \underbrace{\pmb{P}_{12}\left(  -\omega_{0}^{\times} \Bar{\pmb{J}}_{0} \omega_{0}-\sum_{i=1}^{N} m_{i}\rho_{i}^{\times} \omega_{0}^{\times} v_{0}+\sum_{i=1}^{N} \rho_{i}^{\times} \pmb{R}_{0}^{T}m_{i} g \mathbb{k} \right)}_\text{a function of state $\overline{f(X)}$} \\ &+ \underbrace{\pmb{P}_{11}\sum_{i=1}^{N} \pmb{R}_{0}^{T} f_{i} + \pmb{P}_{12}\sum_{i=1}^{N} \rho_{i}^{\times} \pmb{R}_{0}^{T}f_{i}}_\text{a function of state and control input}
    \end{split} \label{eq:expand_v_0}
\end{align}
Thus, the relative degree of the output $y_0$ is $2$, as $r_0$ must be differentiated twice, so that the control inputs appear explicitly as in Eq. (\ref{eq:expand_v_0}).
\begin{align}
    \text{Let} \ \ \pmb{R}_0^Tf_i &= \mathbb{u}_i \implies \sum_{i=1}^N\pmb{R}_0^Tf_i = \sum_{i=1}^N\mathbb{u}_i \\
    &= \underbrace{\left[ \pmb{I}_3 \ \hdots \ \pmb{I}_3 \right]}_\text{$\pmb{S}_1$} \left[ \begin{array}{l}
         \mathbb{u}_1  \\
         \vdots \\
         \mathbb{u}_N
    \end{array} \right] = \pmb{S}_1\bar{\mathbb{u}} \\ 
    \text{Also, } \ &\sum_{i=1}^N \rho_i^\times \pmb{R}_0^Tf_i = \sum_{i=1}^N\rho_i^\times \mathbb{u}_i \\
    \implies & \underbrace{\left[\rho_1^\times \ \hdots \ \rho_N^\times \right]}_\text{$\pmb{S}_2$}\left[\begin{array}{l}
         \mathbb{u}_1 \\
         \vdots \\
         \mathbb{u}_N
    \end{array} \right] = \pmb{S}_2\bar{\mathbb{u}}
\end{align}

Thus, Eq. (\ref{eq:expand_v_0}) can be written as:
\begin{align}
    \dot{v}_0 &= \overline{f(X)} + \left(\pmb{P}_{11}\pmb{S}_1 + \pmb{P}_{12}\pmb{S}_2\right)\bar{\mathbb{u}} \\
    \text{Thus, } \ \ddot{r}_0 &= \dot{\pmb{R}}_0v_0 + \pmb{R}_0\dot{v}_0 \\
    \implies \ddot{r}_0 &= \underbrace{\pmb{R}_0\omega_0^\times v_0 + \pmb{R}_0\overline{f(X)}}_\text{$b(X)$} + \underbrace{\pmb{R}_0\left(\pmb{P}_{11}\pmb{S}_1 + \pmb{P}_{12}\pmb{S}_2 \right)}_\text{$\Delta(X)$}\bar{\mathbb{u}} \\
    & \text{Let } \ \bar{\mathbb{u}} = \Delta^{\dagger}(X)\mathbb{v} - b(X) \\
    & \ \ \ \ \ \boxed{\implies \ddot{y}_0 = \left[\begin{array}{l} \ddot{r}_{0_x} \\ \ddot{r}_{0_y} \\ \ddot{r}_{0_z} \end{array}\right] = \left[\begin{array}{l}\mathbb{v}_1 \\ \mathbb{v}_2 \\ \mathbb{v}_3 \end{array}\right] = \mathbb{v}} \label{eq:fl_r0}
\end{align}

\begin{figure}
    \centering
    
    \begin{tikzpicture}[scale=1.2, every node/.style={transform shape}]
        \draw (0.0, 0.0) rectangle (1.8, 0.8);
        \node at (0.9, 0.4) {\scriptsize$\Delta^\dagger(X)\mathbb{v} - b(X)$};
        
        \draw[-latex] (-0.3, 0.2) -- (0.0, 0.2);
        \draw[-latex] (-0.3, 0.4) -- (0.0, 0.4);
        \draw[-latex] (-0.3, 0.6) -- (0.0, 0.6);
        
        \node at (-0.45, 0.6) {\scriptsize $\mathbb{v}_1$};
        \node at (-0.45, 0.4) {\scriptsize $\mathbb{v}_2$};
        \node at (-0.45, 0.2) {\scriptsize $\mathbb{v}_3$};
        
        \draw[-latex] (1.8, 0.4) -- (2.5, 0.4);
        \node at (2.1, 0.5) {\scriptsize $\bar{\mathbb{u}}$};
        
        \draw (2.5, 0.0) rectangle (4.0, 0.8);
        \node at (3.25, 0.6) {\scriptsize Thrust};
        \node at (3.25, 0.375) {\scriptsize Vectoring};
        \node at (3.25, 0.2) {\scriptsize Control};
        
        \draw[-latex] (4.0, 0.4) -- (5.0, 0.4);
        \node at (4.5, 0.55) {\scriptsize $\mathbb{U}$};
        
        \draw (5.0, 0.0) rectangle (6.5, 0.8);
        \node at (5.75, 0.5) {\scriptsize Payload-UAV};
        \node at (5.75, 0.25) {\scriptsize System};
        
        \draw[-latex] (5.75, 0.0) -- (5.75, -0.5) -- (0.9, -0.5) -- (0.9, 0.0);
        \node at (3.25, -0.35) {\scriptsize State $X$};
        
    \end{tikzpicture}
    
    \caption{Schematic diagram of the feedback linearization for the payload-UAV system.}
    \label{fig:sch_diag_fl_pl_UAV}
\end{figure}
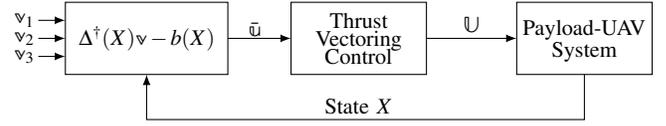

where, $\Delta^{\dagger}(X) \in \mathbb{R}^{3N \times 3}$ denotes the \textit{Moore-Penrose pseudoinverse} of the matrix $\Delta(X) \in \mathbb{R}^{3 \times 3N}$, $b(X) \in \mathbb{R}^3$, $\mathbb{u}_i \in \mathbb{R}^3$, $\bar{\mathbb{u}} \in \mathbb{R}^{3N}$, $\pmb{S}_1, \pmb{S}_2 \in \mathbb{R}^{3\times 3N}$, $\mathbb{v} \in \mathbb{R}^3$ and $\overline{f(X)} \in \mathbb{R}^3$.
\medskip
\begin{theorem}
The Moore-Penrose pseudoinverse $\Delta(X)^\dagger$ of the matrix $\Delta(X)$ always exists, such that $\Delta(X)\Delta^\dagger(X) = \pmb{I}_3$ 
\end{theorem}
\begin{proof}
From definition, $\Delta(X) = \pmb{R}_0\left( \pmb{P}_{11}\pmb{S}_1 + \pmb{P}_{12}\pmb{S}_2 \right)$. The matrix $\pmb{R}_0$ is a rotation matrix that is invertible for any given $X$, thus $rank(\pmb{R}_0) = 3$. The matrix $\pmb{P}_{11}$ is a diagonal matrix with same elements, and $\pmb{S}_{1}$ is a horizontal stack of $N$ identity matrices $\pmb{I}_3$. Thus, $rank(\pmb{P}_{11}\pmb{S}_1) = 3$. The matrix $\pmb{P}_{12}$ is a skew-symmetric matrix, as the vectors $\rho_i$ are all unique\footnote{The full matrix \pmb{P} for the system parameters given in Table \ref{tab:paramVals} can be found \href{https://drive.google.com/drive/folders/1ASU1YWjOVZp5kzjpX9h4WE2W88JTJczX?usp=sharing}{here} or in \cite{nish2022supp_mat} for reference.}, and the matrix $\pmb{S}_2$ is a horizontal stack of all the unique skew-symmetric matrices $\rho_i^\times$. Thus, the product $\pmb{P}_{12}\pmb{S}_2$ is a horizontal stack of skew-symmetric matrices, as the product of two skew-symmetric matrices is a skew-symmetric matrix. Now, consider the diagonal matrix $\pmb{D} \in \mathbb{R}^{3\times 3}$ with same elements and a skew-symmetric matrix that results from the hat-map on a vector $a = \left[a_1 \ a_2 \ a_3 \right]$:
\begin{align*}
\begin{aligned}
    \pmb{D} & = \left[ \begin{array}{ccc} d & 0 & 0 \\ 0 & d & 0 \\ 0 & 0 & d \end{array} \right] \ \text{and} \ \pmb{A} = a^\times = \left[ \begin{array}{ccc} 0 & -a_3 & a_2 \\ a_3 & 0 & -a_1 \\ -a_2 & a_1 & 0  \end{array} \right] \\
    & \implies det(\pmb{D} + \pmb{A}) = d\left(d^2 + a_1^2 + a_2^2 + a_3^2 \right) \neq 0 \\ & \ \ \ \ \ \ \text{for $d\neq 0$ and any a}. \text{ Thus, } rank(\pmb{D} + \pmb{A}) = 3.
\end{aligned}
\end{align*}
Moreover,
\begin{align*}
    & \text{Let } \pmb{K} = \left[ \pmb{D} \ \hdots \ \pmb{D} \right] + \left[ a_1^\times \ \hdots \ a_N^\times \right] = \left[ \pmb{D} + a_1^\times \ \hdots \ \pmb{D} + a_N^\times \right] \\
    & \implies rank(\pmb{K}) = 3.
\end{align*}
By taking individual diagonal matrices stacked horizontally in $\pmb{P}_{11}\pmb{S}_1$ as $\pmb{D}$ and individual skew symmetric matrices stacked horizontally in $\pmb{P}_{12}\pmb{S}_2$ as $\pmb{A}$,  the above steps can be used to prove that $rank(\pmb{P}_{11}\pmb{S}_1 + \pmb{P}_{12}\pmb{S}_2) = 3$. Thus,
\begin{align*}
    & rank(\pmb{R}_0(\pmb{P}_{11}\pmb{S}_1 + \pmb{P}_{12}\pmb{S}_2)) = 3, \text{ since} \\
    & rank(\pmb{R}_0\pmb{K}) = rank(\pmb{K}) \text{ when $\pmb{R}_0$ is full rank}.
\end{align*}
Thus, the matrix $\Delta(X)$ consists of $3$ \textit{linearly independent rows}, which implies that $\Delta(X)$ has a \textit{right inverse} $\Delta(X)^\dagger$ such that $\Delta(X)\Delta(X)^\dagger = \pmb{I}_3$. Thus there exists a bijection between the feedback linearized acceleration dynamics of Eq. (\ref{eq:fl_r0}) and the original acceleration dynamics of Eq. (\ref{eq:expand_v_0}).
\end{proof}

The requirement that the attachment points $\rho_i$ be all unique is a physical requirement, so that there is some clearance between the UAVs themselves. 

\begin{figure*}
    \begin{tikzpicture}
    
        \node[name path=A, trapezium, draw, minimum width=3.5cm, trapezium left angle=120, trapezium right angle=60, line width=0mm, fill=brown!15] at (2.5,-3) {};
        \node(Y)[name path=B, trapezium, draw, minimum width=4cm, trapezium left angle=120, trapezium right angle=60, line width=0mm] at (2.5,-3) {};
        
        \draw (3.3, -2.7) -- (3.3, -2.2);
        \draw (1.05, -2.7) -- (3.3, -2.7) -- (3.95, -3.8);
        
        \draw[name path=C] (0.5, -2.08) -- (1.57, -3.93) -- (4.5, -3.93);
        \draw[name path=D] (0.5, -2.8) -- (1.57, -4.5) -- (4.5, -4.5);
        
        \draw (0.5, -2.08) -- (0.5, -2.8);
        \draw (1.57, -4.5) -- (1.57, -3.93);
        \draw (4.5, -4.5) -- (4.5, -3.93);
        
        \tikzfillbetween[of=A and B, on layer=ft]{brown, opacity=0.5};
        \tikzfillbetween[of=C and D, on layer=ft]{brown, opacity=0.5};
        
        \draw[name path=E] (0.55, -2.9) -- (0.55, 1.0); 
        \draw[name path=F] (0.6, -3.0) -- (0.6, 0.95);
        \tikzfillbetween[of=E and F, on layer=ft]{black, opacity=0.7};
        
        \draw[name path=G] (1.5, -4.4) -- (1.5, -1.0);
        \draw[name path=H] (1.45, -4.3) -- (1.45, -0.95);
        \tikzfillbetween[of=G and H, on layer=ft]{black, opacity=0.7};
        
        \draw[name path=I] (4.5, -3.95) -- (4.5, -1.0);
        \draw[name path=J] (4.45, -3.85) -- (4.45, -0.95);
        \tikzfillbetween[of=I and J, on layer=ft]{black, opacity=0.7};
    
        \draw[name path=K] (3.55, -2.3) -- (3.55,0.95);
        \draw[name path=L] (3.5, -2.2) -- (3.5, 1.0);
        \tikzfillbetween[of=K and L, on layer=ft]{black, opacity=0.7};
        
        \node[cylinder, draw=black!40, shape border rotate=90, minimum width=1.7cm, minimum height=0.8cm, cylinder uses custom fill, cylinder body fill = black!40, cylinder end fill = black!60] (c) at (2.5,-3.3) {\tiny Load };
        
        \node[ellipse, draw, fill = gray!60, minimum width = 0.55cm, minimum height = 0.08cm] (prop11) at (0.1,1.15) {};
        \node[ellipse, draw, fill = gray!60, minimum width = 0.55cm, minimum height = 0.08cm] (prop12) at (1.0,1.15) {};
        \node[ellipse, draw, fill = gray!60, minimum width = 0.55cm, minimum height = 0.08cm] (prop13) at (1.0,1.7) {};
        \node[ellipse, draw, fill = gray!60, minimum width = 0.55cm, minimum height = 0.08cm] (prop14) at (0.1,1.7) {};
        \draw[line width=0mm] (prop11) -- (prop13);
        \draw[line width=0mm] (prop12) -- (prop14);
        
        \node[ellipse, draw, fill = gray!60, minimum width = 0.55cm, minimum height = 0.08cm] (prop21) at (1.05,-0.8) {};
        \node[ellipse, draw, fill = gray!60, minimum width = 0.55cm, minimum height = 0.08cm] (prop22) at (1.95,-0.8) {};
        \node[ellipse, draw, fill = gray!60, minimum width = 0.55cm, minimum height = 0.08cm] (prop23) at (1.8,-0.25) {};
        \node[ellipse, draw, fill = gray!60, minimum width = 0.55cm, minimum height = 0.08cm] (prop24) at (1.0,-0.25) {};
        \draw[line width=0mm] (prop21) -- (prop23);
        \draw[line width=0mm] (prop22) -- (prop24);
        
        \node[ellipse, draw, fill = gray!60, minimum width = 0.55cm, minimum height = 0.08cm] (prop31) at (1.05 + 3,-0.8) {};
        \node[ellipse, draw, fill = gray!60, minimum width = 0.55cm, minimum height = 0.08cm] (prop32) at (1.95 + 3,-0.8) {};
        \node[ellipse, draw, fill = gray!60, minimum width = 0.55cm, minimum height = 0.08cm] (prop33) at (1.8 + 3,-0.25) {};
        \node[ellipse, draw, fill = gray!60, minimum width = 0.55cm, minimum height = 0.08cm] (prop34) at (1.0 + 3,-0.25) {};
        \draw[line width=0mm] (prop31) -- (prop33);
        \draw[line width=0mm] (prop32) -- (prop34);
        
        \node[ellipse, draw, fill = gray!60, minimum width = 0.55cm, minimum height = 0.08cm] (prop11) at (0.1 + 3,1.15) {};
        \node[ellipse, draw, fill = gray!60, minimum width = 0.55cm, minimum height = 0.08cm] (prop12) at (1.0 + 3,1.15) {};
        \node[ellipse, draw, fill = gray!60, minimum width = 0.55cm, minimum height = 0.08cm] (prop13) at (1.0 + 3,1.7) {};
        \node[ellipse, draw, fill = gray!60, minimum width = 0.55cm, minimum height = 0.08cm] (prop14) at (0.1 + 3,1.7) {};
        \draw[line width=0mm] (prop11) -- (prop13);
        \draw[line width=0mm] (prop12) -- (prop14);
        
        \draw[<-, color=red!70, densely dashed] (0.55, 1.0) -- (0.55, 1.4);
        \draw[<-, color=red!70, densely dashed] (1.5, -1.0) -- (1.5, -0.6);
        \draw[<-, color=red!70, densely dashed] (4.5, -1.0) -- (4.5, -0.6);
        \draw[<-, color=red!70, densely dashed] (3.55, 0.95) -- (3.55, 1.35);
        
        \draw[blue, thick] (-1.6, 2.5) to[out=30, in=150] (0.4, 2.0);
        \draw[blue, thick] (-1.6 + 0.8, 2.5 - 2.) to[out=30, in=150] (0.4 + 0.8, 2.0 - 2.);
        \draw[blue, thick] (-1.6 + 3.8, 2.5 - 2.) to[out=30, in=150] (0.4 + 3.8, 2.0 - 2.);
        \draw[blue, thick] (-1.6 + 2.8, 2.5) to[out=30, in=150] (0.4 + 2.8, 2.0);
        
        \node at (3.0, -5.0) {(a)};
        
        \tikzset{shift={(8,2)}}
        
        \node[name path=A, trapezium, draw, minimum width=3.5cm, trapezium left angle=120, trapezium right angle=60, line width=0mm, fill=brown!15] at (2.5,-3) {};
        \node(Y)[name path=B, trapezium, draw, minimum width=4cm, trapezium left angle=120, trapezium right angle=60, line width=0mm] at (2.5,-3) {};
        
        \draw (3.3, -2.7) -- (3.3, -2.2);
        \draw (1.05, -2.7) -- (3.3, -2.7) -- (3.95, -3.8);
        
        \draw[name path=C] (0.5, -2.08) -- (1.57, -3.93) -- (4.5, -3.93);
        \draw[name path=D] (0.5, -2.8) -- (1.57, -4.5) -- (4.5, -4.5);
        
        \draw (0.5, -2.08) -- (0.5, -2.8);
        \draw (1.57, -4.5) -- (1.57, -3.93);
        \draw (4.5, -4.5) -- (4.5, -3.93);
        
        \tikzfillbetween[of=A and B, on layer=ft]{brown, opacity=0.5};
        \tikzfillbetween[of=C and D, on layer=ft]{brown, opacity=0.5};
        
        \draw[name path=E] (0.55, -2.9) -- (0.55, 1.0); 
        \draw[name path=F] (0.6, -3.0) -- (0.6, 0.95);
        \tikzfillbetween[of=E and F, on layer=ft]{black, opacity=0.7};
        
        \draw[name path=G] (1.5, -4.4) -- (1.5, -1.0);
        \draw[name path=H] (1.45, -4.3) -- (1.45, -0.95);
        \tikzfillbetween[of=G and H, on layer=ft]{black, opacity=0.7};
        
        \draw[name path=I] (4.5, -3.95) -- (4.5, -1.0);
        \draw[name path=J] (4.45, -3.85) -- (4.45, -0.95);
        \tikzfillbetween[of=I and J, on layer=ft]{black, opacity=0.7};
    
        \draw[name path=K] (3.55, -2.3) -- (3.55,0.95);
        \draw[name path=L] (3.5, -2.2) -- (3.5, 1.0);
        \tikzfillbetween[of=K and L, on layer=ft]{black, opacity=0.7};
        
        \node[cylinder, draw=black!40, shape border rotate=90, minimum width=1.7cm, minimum height=0.8cm, cylinder uses custom fill, cylinder body fill = black!40, cylinder end fill = black!60] (c) at (2.5,-3.3) {\tiny load};
        
        \node[ellipse, draw, fill = gray!60, minimum width = 0.55cm, minimum height = 0.08cm] (prop11) at (0.1,1.15 - 0.45) {};
        \node[ellipse, draw, fill = gray!60, minimum width = 0.55cm, minimum height = 0.08cm] (prop12) at (1.0,1.15 - 0.45) {};
        \node[ellipse, draw, fill = gray!60, minimum width = 0.55cm, minimum height = 0.08cm] (prop13) at (1.0,1.7 - 0.45) {};
        \node[ellipse, draw, fill = gray!60, minimum width = 0.55cm, minimum height = 0.08cm] (prop14) at (0.1,1.7 - 0.45) {};
        \draw[line width=0mm] (prop11) -- (prop13);
        \draw[line width=0mm] (prop12) -- (prop14);
        
        \node[ellipse, draw, fill = gray!60, minimum width = 0.55cm, minimum height = 0.08cm] (prop21) at (1.05,-0.8 - 0.45) {};
        \node[ellipse, draw, fill = gray!60, minimum width = 0.55cm, minimum height = 0.08cm] (prop22) at (1.95,-0.8 - 0.45) {};
        \node[ellipse, draw, fill = gray!60, minimum width = 0.55cm, minimum height = 0.08cm] (prop23) at (1.8,-0.25 - 0.45) {};
        \node[ellipse, draw, fill = gray!60, minimum width = 0.55cm, minimum height = 0.08cm] (prop24) at (1.0,-0.25 - 0.45) {};
        \draw[line width=0mm] (prop21) -- (prop23);
        \draw[line width=0mm] (prop22) -- (prop24);
        
        \node[ellipse, draw, fill = gray!60, minimum width = 0.55cm, minimum height = 0.08cm] (prop31) at (1.05 + 3,-0.8 - 0.45) {};
        \node[ellipse, draw, fill = gray!60, minimum width = 0.55cm, minimum height = 0.08cm] (prop32) at (1.95 + 3,-0.8 - 0.45) {};
        \node[ellipse, draw, fill = gray!60, minimum width = 0.55cm, minimum height = 0.08cm] (prop33) at (1.8 + 3,-0.25 - 0.45) {};
        \node[ellipse, draw, fill = gray!60, minimum width = 0.55cm, minimum height = 0.08cm] (prop34) at (1.0 + 3,-0.25 - 0.45) {};
        \draw[line width=0mm] (prop31) -- (prop33);
        \draw[line width=0mm] (prop32) -- (prop34);
        
        \node[ellipse, draw, fill = gray!60, minimum width = 0.55cm, minimum height = 0.08cm] (prop11) at (0.1 + 3,1.15 - 0.45) {};
        \node[ellipse, draw, fill = gray!60, minimum width = 0.55cm, minimum height = 0.08cm] (prop12) at (1.0 + 3,1.15 - 0.45) {};
        \node[ellipse, draw, fill = gray!60, minimum width = 0.55cm, minimum height = 0.08cm] (prop13) at (1.0 + 3,1.7 - 0.45) {};
        \node[ellipse, draw, fill = gray!60, minimum width = 0.55cm, minimum height = 0.08cm] (prop14) at (0.1 + 3,1.7 - 0.45) {};
        \draw[line width=0mm] (prop11) -- (prop13);
        \draw[line width=0mm] (prop12) -- (prop14);
        
        \draw[densely dashed, red] (3.25, -4.2) to[out=-20, in=-180] (6, -6.2);
        \node at (4.8, -5.0) {$r_{0_d}$};
        
        \node at (3.0, -7.0) {(b)};

    \end{tikzpicture}
    
    \caption{Figure demonstrates the different phases of payload transportation: a)UAVs come to a formation and hover above the links and vertically descend to attach. b)Trajectory tracking of the payload-UAV system.}
    \label{fig:formation_UAV_payload}
    
\end{figure*}
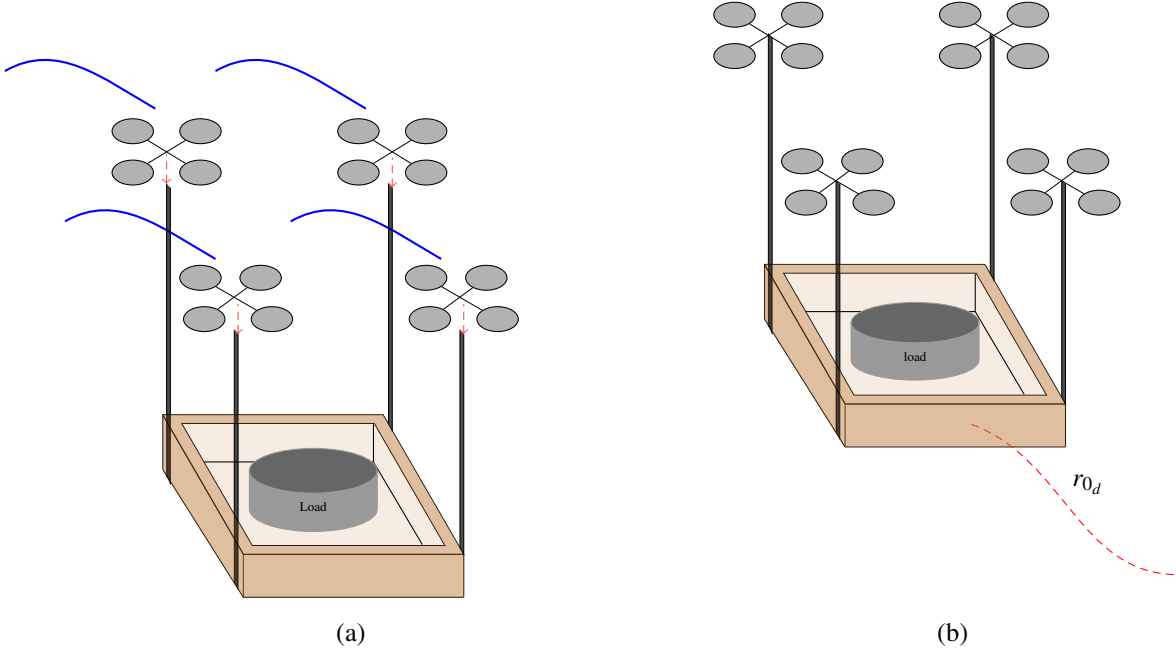

\subsection{Exponentially Stable Tracking Control Law}
Here, an \textit{exponentially stable tracking control law} is designed for the feedback linearized dynamics of both the UAV and the payload-UAV system characterized by Eq. (\ref{eq:fl_uav_dyn}) and Eq. (\ref{eq:fl_r0}) respectively. 

\begin{theorem}
\label{thm:uav_position_tracking}
Let $y_d = \left[r_{x_d} \ r_{y_d} \ r_{z_d} \ \psi_d \right]^T$ be the desired reference trajectory that the UAV must track. Then, globally exponentially stable tracking is achieved under the feedback control law:
\begin{subequations}
\begin{empheq}[box=\widefbox]{align}
    v_1 & = \ddddot{r}_{x_d} + \beta_1\dddot{e}_{x} + \beta_2\ddot{e}_x + \beta_3\dot{e}_x + \beta_4 e_x \\
    v_2 & = \ddddot{r}_{y_d} + \beta_1\dddot{e}_{y} + \beta_2\ddot{e}_y + \beta_3\dot{e}_y + \beta_4 e_y \\
    v_3 & = \ddddot{r}_{z_d} + \beta_1\dddot{e}_{z} + \beta_2\ddot{e}_z + \beta_3\dot{e}_z + \beta_4 e_z \\
    v_4 & = \ddot{\psi}_d + \beta_5\dot{e}_{\psi} + \beta_6 e_\psi
\end{empheq}
\label{eq:ctrl_law_uav}
\end{subequations}
\end{theorem}
\begin{proof}
(For x-position) Let $e_x = r_{d_x} - r_x$ be the tracking error. Define $q_x = \dddot{e}_x + \alpha_1\ddot{e}_x + \alpha_2\dot{e}_x + \alpha_3e_x$. A candidate Lyapunov function can be chosen as $V_x = \frac{1}{2}q_x^2$ which is \textit{globally positive definite} for $q_x \neq 0$. For exponential stability, the Lyapunov function must satisfy the differential equation
\begin{align*}
    &\dot{V}_x = -k_xV_x, \ k\in \mathbb{R}^+ \implies \dot{q}_xq_x = -\frac{k_x}{2}q_x^2 \\ 
    \implies &\ddddot{e}_x + \alpha_1\dddot{e}_x + \alpha_2\ddot{e}_x + \alpha_3\dot{e}_x = \frac{-k_x}{2}\left( \dddot{e}_x + \alpha_1\ddot{e}_x + \alpha_2\dot{e}_x + \alpha_3e_x \right) \\
    \implies & v_1 = \ddddot{r}_{x_d} + (\alpha_1+k_x/2)\dddot{e}_x + (\alpha_2+k_x\alpha_1/2)\ddot{e}_x \\ & \ \ \ \ \ \ \ \ \ \ \ + (\alpha_3 + k_x\alpha_2/2)\dot{e}_x  + (k_x\alpha_3/2)e_x \\
    \implies & v_1 = \ddddot{r}_{x_d} + \beta_1\dddot{e}_{x} + \beta_2\ddot{e}_x + \beta_3\dot{e}_x + \beta_4 e_x
\end{align*}
A similar proof for (y-, z-) position can be derived. For ($\psi$)-axis, let $e_\psi = \psi_d - \psi$ be the yaw tracking error. Define $q_\psi = \dot{e}_\psi + \alpha_4e_\psi$. A candidate Lyapunov function can be chosen as $V_\psi = \frac{1}{2}q_\psi^2$ which is \textit{globally positive definite} for $q_\psi \neq 0$. For exponential stability, the Lyapunov function must satisfy the differential equation
\begin{align*}
    &\dot{V}_\psi = -k_\psi V_\psi, \ k_\psi \in \mathbb{R}^+ \implies \dot{q}_\psi q_\psi = -\frac{k_\psi}{2}q_\psi^2 \\
    \implies &\ddot{e}_\psi + \alpha_4\dot{e}_\psi = -\frac{k_\psi}{2}\left( \dot{e}_\psi + \alpha_4 e_\psi \right) \\
    \implies &v_4 = \ddot{\psi}_d + (\alpha_4 + k_\psi/2)\dot{e}_\psi + (k_\psi\alpha_4 / 2)e_\psi \\
    \implies &v_4 = \ddot{\psi}_d + \beta_5\dot{e}_\psi + \beta_6 e_\psi
\end{align*}
\end{proof}

\begin{theorem}
Let $y_{0_d} = r_{0_d} = \left[ r_{0_{x_d}} \ r_{0_{y_d}} \ r_{0_{z_d}} \right]^T$ be the desired reference trajectory that the payload must track. Then, globally exponentially stable tracking is achieved under the feedback control law:
\begin{subequations}
\begin{empheq}[box=\widefbox]{align}
    \mathbb{v}_1 &= \ddot{r}_{0_{x_d}} + \beta_7\dot{e}_{r_x} + \beta_8 e_{r_x} \\
    \mathbb{v}_2 &= \ddot{r}_{0_{y_d}} + \beta_7\dot{e}_{r_y} + \beta_8 e_{r_y} \\
    \mathbb{v}_3 &= \ddot{r}_{0_{z_d}} + \beta_7\dot{e}_{r_z} + \beta_8 e_{r_z}
\end{empheq}
\label{eq:ctrl_law_pl_uav}
\end{subequations}
\begin{proof}
Define the errors $e_{r_\mathbb{I}} = r_{0_{\mathbb{I}}} - r_{0_\mathbb{I}}$ where $\mathbb{I}$ can be $x, y, z$ Define $q_{0_\mathbb{I}} = \ddot{e}_{r_\mathbb{I}} + \alpha_5\dot{e}_{r_\mathbb{I}} + \alpha_6 e_{r_\mathbb{I}}$. Select a candidate Lyapunov function to be $V_{r_\mathbb{I}} = \frac{1}{2}q^2_{r_\mathbb{I}}$. The rest of the proof is similar to the proof given in Theorem \ref{thm:uav_position_tracking} (yaw angle tracking).
\end{proof}
\end{theorem}

Once the control inputs $v$ for the UAV and $\mathbb{v}$ for the payload-UAV system is found using the above tracking laws, the original control inputs $\bar{U}$ for the UAV and $\bar{\mathbb{u}}$ for the payload-UAV system is found from the feedback linearized bijective map as discussed in Section \ref{sec:input-output_fbl_sub_a}.

\subsection{Thrust-Vectoring Control}
Once $\bar{\mathbb{u}}$ is found for the payload-UAV system, one must transform this to the original control input $\mathbb{U}$ of Eq. (\ref{eq:payload_uav_u}) for the payload-UAV system. However, recall that 
\begin{align}
    \bar{\mathbb{u}} = \left[ \mathbb{u}_1^T \ \hdots \ \mathbb{u}_N^T \right]^T, \ \text{and } \ \mathbb{u}_i = \pmb{R}_0^Tf_i = f_{t_i}\pmb{R}_0^T\pmb{R}_i\mathbb{k} \label{eq:relation_u_f}
\end{align}
From $\mathbb{u}_i$, the quantities $\left[ f_{t_i} \ \tau_{x_i} \ \tau_{y_i} \ \tau_{z_i} \right]$ must be calculated. The main challenge here is to find the desired attitude of the UAVs from Eq. (\ref{eq:relation_u_f}), based on the direction of the control command vector $\mathbb{u}_i$. For this purpose, the \textit{thrust vectoring control approach} is used as described in \cite{silva2018quadrotor}. The whole idea of thrust vectoring control is to find the desired orthonormal axes from the direction of $\mathbb{u}_i$ and thus the desired rotation matrix, from which the desired Euler angles can be computed easily. Then an attitude controller can be designed to find the torque inputs $\left[ \tau_{x_i} \ \tau_{y_i} \ \tau_{z_i} \right]$. From Eq. (\ref{eq:relation_u_f}), the thrust force $f_{t_i}$ can be computed easily as follows:
\begin{align}
    \mathbb{u_i} =& \pmb{R}_0^Tf_i = f_{t_i}\pmb{R}_0^T\pmb{R}_i\mathbb{k} \\
    \implies \Vert\mathbb{u}_i\Vert_2 =& \Vert f_{t_i}\pmb{R}_0^T\pmb{R}_i\mathbb{k} \Vert_2 \\
    =& f_{t_i}
\end{align}
where $\Vert . \Vert_2$ denotes the \textit{$l^2$ norm}. To find the desired Euler angles, find the desired directions of the orthonormal body-fixed axes $\mathbb{x}^B_{d_i}, \ \mathbb{y}^B_{d_i}, \ \mathbb{z}^B_{d_i}$ of the $i^{th}$ UAV as:
\begin{align*}
    \mathbb{z}^B_{d_i} = \frac{\mathbb{u}_i}{\Vert \mathbb{u}_i \Vert_2}& \\
    \text{Desirable yaw angle is $\psi_{d_i} = 0$}& \ \implies \mathbb{x}^\psi = \left[ 1 \ 0 \ 0 \right]^T \\ 
    \mathbb{y}^B_{d_i} = \frac{\mathbb{z}^B_{d_i} \times \mathbb{x}^\psi}{\Vert \mathbb{z}^B_{d_i} \times \mathbb{x}^\psi \Vert_2} \ \ \text{and} \ \ &\mathbb{x}^B_{d_i} = \frac{\mathbb{y}^B_{d_i} \times \mathbb{z}^B_{d_i}}{\Vert \mathbb{y}^B_{d_i} \times \mathbb{z}^B_{d_i} \Vert_2}
\end{align*}
where $\mathbb{x}^\psi$ is the intermediate orientation of the first reference frame rotation according to the ZYX Euler angle definition. The desired yaw angle $\psi_{d_i}$ of the $i^{th}$ UAV is 0, due to the 2-DOF spherical joint constraint. The desired rotation matrix $\pmb{R}_{d_i}$ of the $i^{th}$ UAV and thus the desired roll and pitch angles can be calculated as:
\begin{align}
    \pmb{R}_{d_i} &= \left[ \mathbb{x}^B_{d_i} \ \ \mathbb{y}^B_{d_i} \ \ \mathbb{z}^B_{d_i}  \right] \\
    \phi_{d_i} = \text{atan2}\left( \frac{\pmb{R}_{d_i, 32}}{\pmb{R}_{d_i, 33}} \right) \ &\text{and} \ \theta_{d_i} = \text{atan2}\left( \frac{-\pmb{R}_{d_i, 31}}{\sqrt{ \pmb{R}^2_{d_i, 32} + \pmb{R}^2_{d_i, 33} }} \right)
\end{align}
Since the attitude dynamics of the UAV is decoupled from the rest of the system, Eq. (\ref{eq:w_i}) can be feedback linearized into:
\begin{align}
    \pmb{J}_i\dot{\omega}_i &= -\omega_i^\times\pmb{J}_i\omega_i + \tau_i \\ 
    \text{Let} \ \ \tau_i &= \pmb{J}_i\overline{\tau_i} + \omega^\times_i\pmb{J}_i\omega_i \label{eq:fl_tau} \\
    \implies & \overline{\tau_i} = \dot{\omega}_i \\
    \text{Now,} \ \ \dot{\Theta}_i = \pmb{R}_{\Theta}\omega_i &\implies \ddot{\Theta}_i = \pmb{R}_\Theta\dot{\omega}_i + \dot{\pmb{R}}_\Theta\omega_i \\
    \implies \ddot{\Theta}_i = \pmb{R}_\Theta\overline{\tau_i} &+ \dot{\pmb{R}}_\Theta\omega_i = \mathbb{T}_i \label{eq:fl_T}
\end{align}
where $\pmb{R}_\Theta$ is the matrix that transforms angular velocities to Euler rates, and $\mathbb{T} = \left[ \mathbb{T}_\phi \ \mathbb{T}_\theta \ \mathbb{T}_\psi \right]^T$ is the new transformed attitude control input.
\begin{theorem}
Given the desired Euler angles $\phi_{d_i}, \theta_{d_i}, \psi_{d_i}$ there exists an exponentially stable attitude tracking controller under the control law:
\begin{subequations}
\begin{empheq}[box=\widefbox]{align}
    \mathbb{T}_{\phi_i} &= \ddot{\phi}_{d_i} + \beta_9\dot{e}_{\phi_i} + \beta_{10}e_{\phi_i} \\
    \mathbb{T}_{\theta_i} &= \ddot{\theta}_{d_i} + \beta_9\dot{e}_{\theta_i} + \beta_{10}e_{\theta_i} \\
    \mathbb{T}_{\psi_i} &= \ddot{\psi}_{d_i} + \beta_9\dot{e}_{\psi_i} + \beta_{10}e_{\psi_i}
\end{empheq}
\label{eq:ctrl_law_pl_uav_att}
\end{subequations}
\end{theorem}
\begin{proof}
Define the tracking errors as $e_{\phi_i} = \phi_{d_i} - \phi_i$, $e_{\theta_i} = \theta_{d_i} - \theta_i$ and $e_{\psi_i} = \psi_{d_i} - \psi_i$ respectively. Let $q_{\phi_i} = \ddot{e_{\phi_i}} + \alpha_9\dot{e}_{\phi_i} + \alpha_{10}e_{\phi_i}$ and $q_{\theta_i}, q_{\psi_i}$ defined similarly. The Lyapunov functions can be chosen as $V_{\phi_i} = \frac{1}{2}q^2_{\phi_i}$, $V_{\theta_i} = \frac{1}{2}q^2_{\theta_i}$ and $V_{\psi_i} = \frac{1}{2}q^2_{\psi_i}$ respectively. The rest of the proof is similar to the proof given in Theorem \ref{thm:uav_position_tracking}
(yaw angle tracking).
\end{proof}
Once $\mathbb{T}_i$ is found, Eq. \ref{eq:fl_T} can be rearranged to find $\overline{\tau_i}$, and Eq. (\ref{eq:fl_tau}) can be used to compute $\tau_i$, which is given to the payload-UAV system along with the total thrust input $f_{t_i}, \ i=1,\hdots,N$.

\begin{table}
    \centering
    \caption{Parameter values for the payload-UAV system }
    \label{tab:paramVals}
    \begin{tabular}{ccccc} \toprule
         & $m$ & $J_{xx}$ & $J_{yy}$ & $J_{zz}$ \\ \midrule
         \text{Payload} & 3.0 & 0.556 & 0.556 & 0.556 \\
         \text{UAVs} & 1.5 & 0.029 & 0.029 & 0.055 \\ \midrule
         & & $\rho_i$ & & $l_i$ \\ \midrule
         \text{UAV1} & & $[0.5 \ 0.5 \ -0.125]^T$ & & 3.2 \\
         \text{UAV2} & & $[0.5 \ \shortminus0.5 \ -0.125]^T$ & & 3.2 \\
         \text{UAV3} & & $[\shortminus0.5 \ \shortminus0.5 \ -0.125]^T$ & & 3.2 \\
         \text{UAV4} & & $[\shortminus0.5 \ 0.5 \ -0.125]^T$ & & 3.2 \\\bottomrule
    \end{tabular}
\end{table}

\section{Simulation Results}
\label{sec:results}

This section evaluates the performance of the proposed nonlinear controller for cooperative payload transport. In particular, a numerical simulation is conducted in \texttt{MATLAB} to validate the tracking performance of the proposed controller. Next, a high-fidelity \texttt{Gazebo} simulation is carried out to verify the real-time integrity of the controller. The controller is implemented in a \texttt{C++} script which communicates to the real-time model spawned in \texttt{Gazebo} via \texttt{ROS C++ APIs}. The model of the payload-UAV system is written in \texttt{urdf} format, and the \texttt{RotorS} package is used in conjunction with this model to simulate the propeller aerodynamics.

\footnotetext{A video demonstration for the \texttt{Gazebo} simulation can be found in \cite{nish2022supp_mat} or here: \href{https://youtu.be/ltCwnNr9nJ0}{https://youtu.be/ltCwnNr9nJ0}}

\subsection{Numerical Evaluation}
The parameters for the payload-UAV model is provided in Table \ref{tab:paramVals} and are same throughout the subsequent discussions. The constants have the values $\beta_1 = 6.5, \ \beta_2 = 26, \ \beta_3 = 28.5, \ \beta_4 = 15, \beta_5 = 4, \ $ $\beta_6 = 4, \ \beta_7 = 4.5, \ \beta_8 = 5, \ \beta_9 = 18$ and $\beta_{10} = 85$ respectively. For trajectory tracking, a sample circular trajectory is to be tracked at a height of $5m$ for both the payload-UAV model of Eq. (\ref{eq:v_0}) - (\ref{eq:w_i}) and the extended UAV model of Eq. (\ref{eq:ext_uav_model}):
\begin{subequations}
\begin{align}
    r_{0_{x_d}}(t) = r_{x_d} =& \ \text{sin}(t / 2) \\
    r_{0_{y_d}}(t) = r_{y_d} =& \ \text{cos}(t / 2) \\
    r_{0_{z_d}}(t) = r_{z_d} =& \ 5u(t)
\end{align}
\end{subequations}
where $u(t)$ is the standard unit step function.
\subsubsection{Single UAV tracking}
The tracking performance of the feedback linearized control law of Eq. (\ref{eq:ctrl_law_uav}) is shown in Fig. (\ref{fig:traj_track_uav}) and Fig. (\ref{fig:att_uav}). It can be seen that the trajectory of the UAV \textit{exponentially converges} to the desired trajectory after which the tracking error is insignificant. The roll and pitch angles vary between $\pm 2^o$ and the yaw angle is zero since $\psi_d = 0$.
\begin{figure}[H]
    \centering
    \includegraphics[scale=0.54]{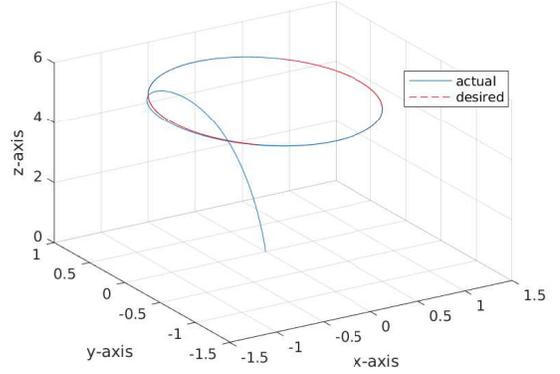}
    \caption{The reference trajectory for a single UAV is shown in red dotted line. The actual trajectory is shown in blue solid line.}
    \label{fig:traj_track_uav}
\end{figure}

\begin{figure}[!htb]
    \centering
    \includegraphics[scale=0.54]{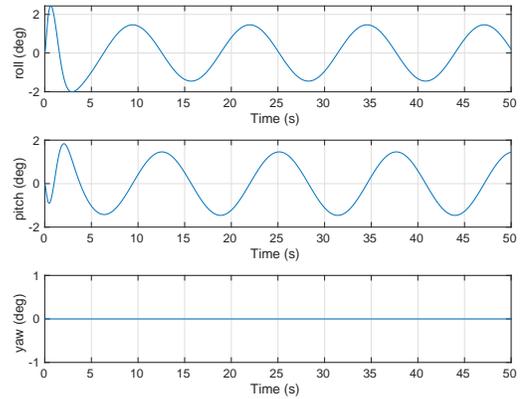}
    \caption{The attitude of the UAV throughout the simulation is shown in blue solid line.}
    \label{fig:att_uav}
\end{figure}
\begin{figure*}
    \centering
    \includegraphics[width=18cm, height=6cm]{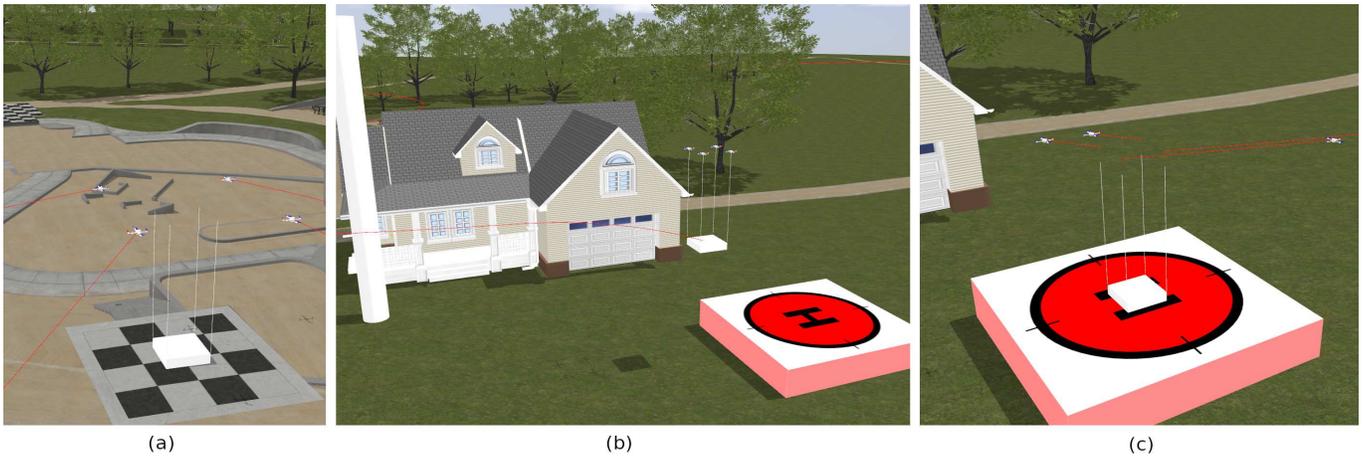}
    \caption[Caption for LOF]{Figure shows the three stages of the simulation: (a) shows the four UAVs that come together in formation to attach themselves to the rigid links. (b) shows the trajectory tracking of the payload-UAV system which has to land on the helipad. (c) shows the UAVs breaking apart and resuming their tasks once the payload lands on the helipad. The trajectories of the UAVs and the payload is highlighted in Red color in the respective figures \protect\footnotemark.}
    \label{fig:pay_uav_stages}
\end{figure*}
\subsubsection{Payload-UAV Tracking}
It is assumed that there are 4 independent UAVs that come together to attach to the rigid payload frame. Thereafter, the entire system dynamics changes to the payload-UAV dynamics as discussed previously. The control law $switches$ to Eq. (\ref{eq:ctrl_law_pl_uav}) from Eq. (\ref{eq:ctrl_law_uav}) and the tracking performance is shown in Fig. (\ref{fig:traj_track_pl_uav}) and the attitudes of the four UAVs are shown in Fig. (\ref{fig:att_uav_1}) - (\ref{fig:att_uav_4}).

\begin{figure}[!htb]
    \centering
    \includegraphics[scale=0.54]{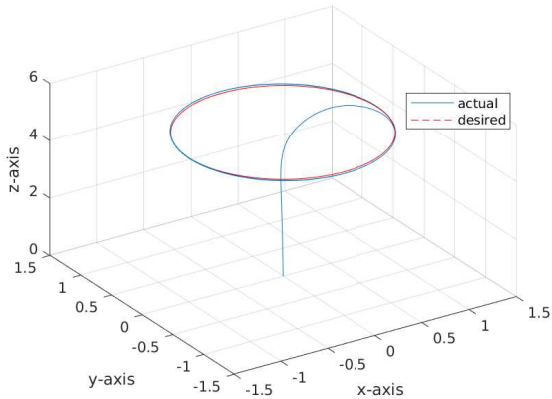}
    \caption{The reference trajectory is shown in red dotted lines, and the actual position of the payload is shown in solid blue line}
    \label{fig:traj_track_pl_uav}
\end{figure}
\begin{figure}[!htb]
    \centering
    \includegraphics[scale=0.54]{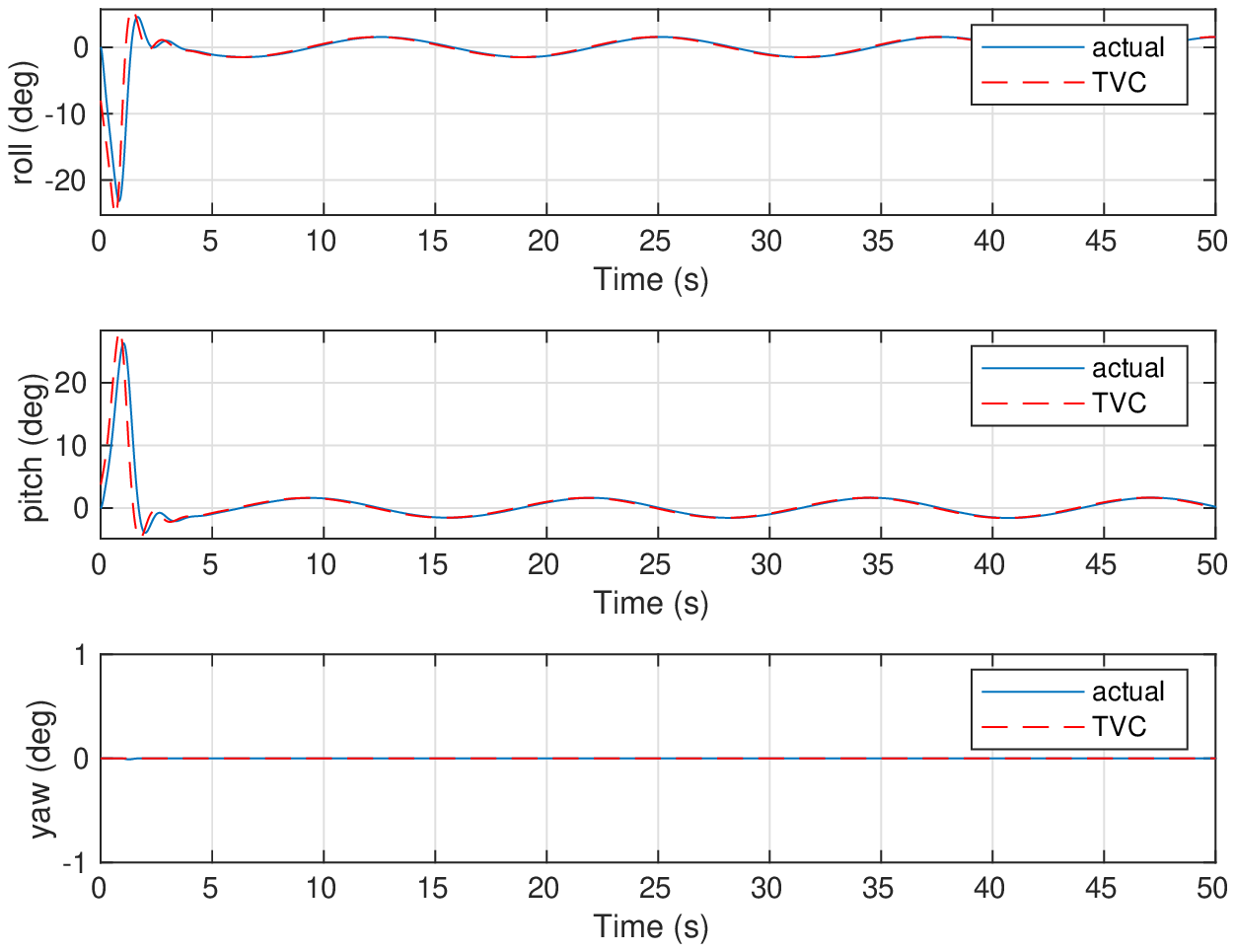}
    \caption{The reference attitude as provided by TVC is shown in red dotted lines. The attitude of UAV 1 is shown in blue solid lines.}
    \label{fig:att_uav_1}
\end{figure}
\begin{figure}[!htb]
    \centering
    \includegraphics[scale=0.54]{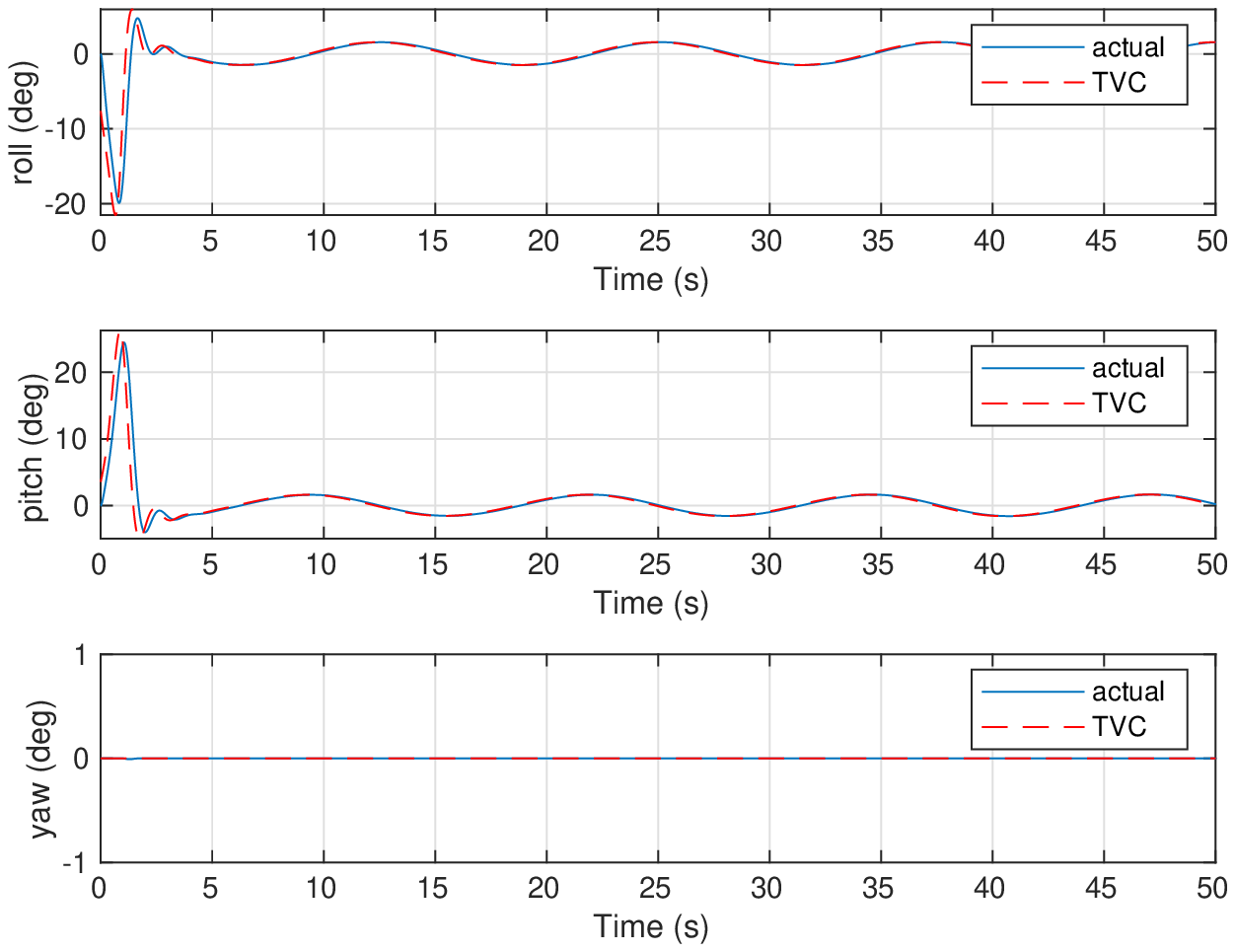}
    \caption{The reference attitude as provided by TVC is shown in red dotted lines. The attitude of UAV 2 is shown in blue solid lines.}
    \label{fig:att_uav_2}
\end{figure}
\begin{figure}[!htb]
    \centering
    \includegraphics[scale=0.54]{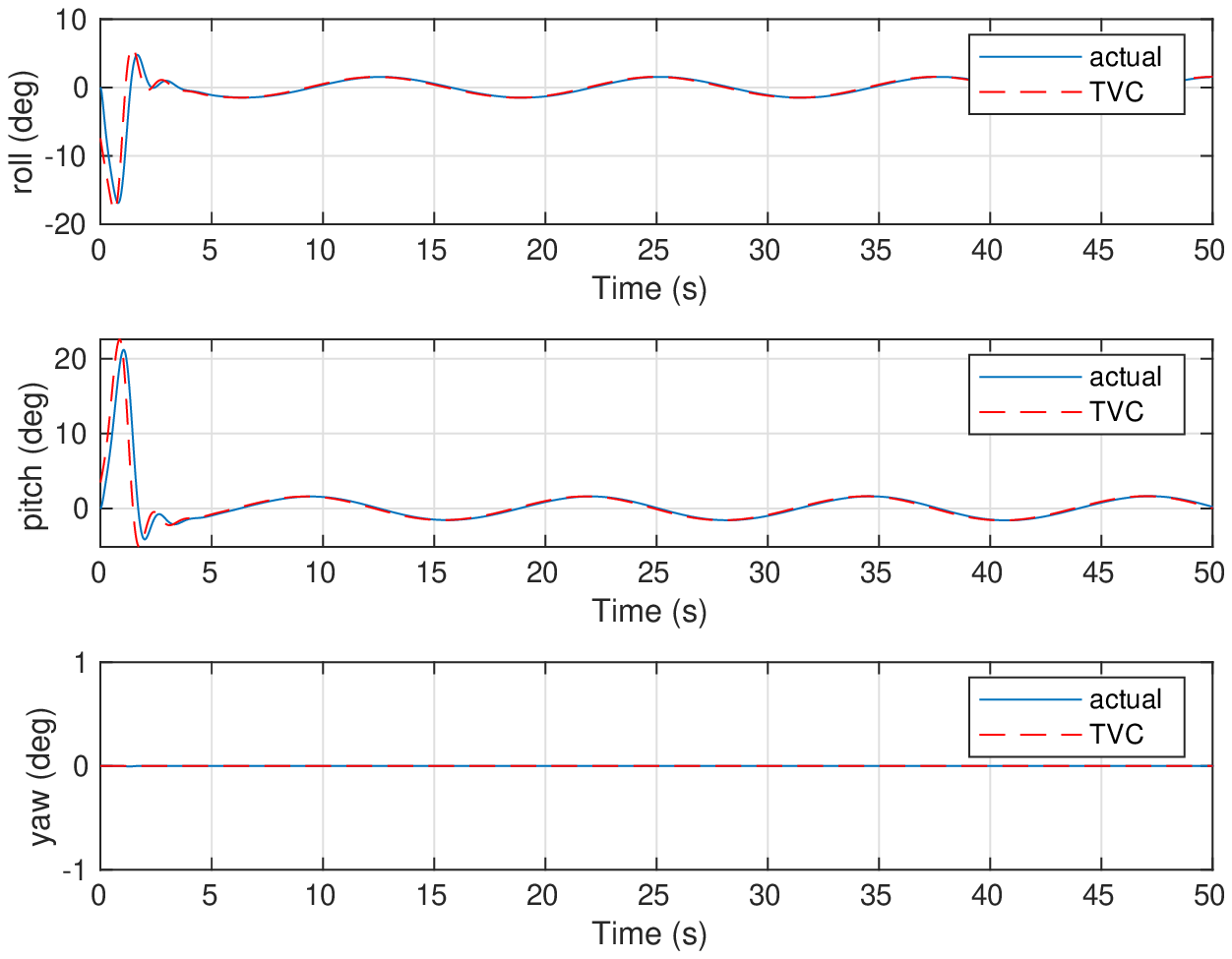}
    \caption{The reference attitude as provided by TVC is shown in red dotted lines. The attitude of UAV 3 is shown in blue solid lines.}
    \label{fig:att_uav_3}
\end{figure}
\begin{figure}[!htb]
    \centering
    \includegraphics[scale=0.54]{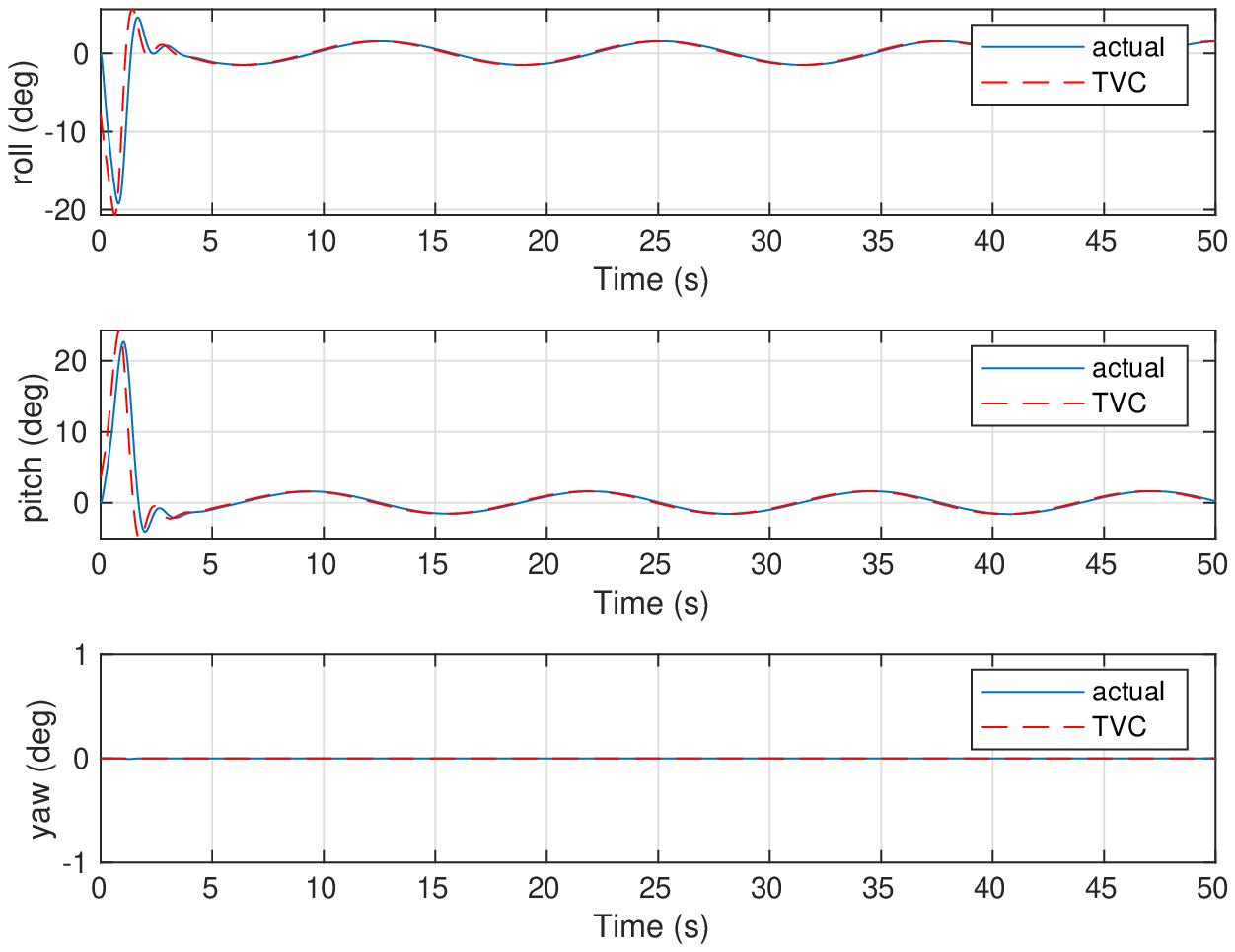}
    \caption{The reference attitude as provided by TVC is shown in red dotted lines. The attitude of UAV 4 is shown in blue solid lines.}
    \label{fig:att_uav_4}
\end{figure}
It is evident from the figures that the tracking controller of Eq. (\ref{eq:ctrl_law_pl_uav}) \textit{exponentially converges} to the desired reference trajectories. The desired output attitudes from the \textit{Thrust Vectoring Controller} (TVC) is shown in the attitude tracking figures of each UAVs. The desired yaw angles for all the UAVs are set to zero due to the physical constraint of the 2-DOF spherical joint attachment. The controller constants $\beta_i$'s are calculated after choosing appropriate values of $\alpha_i$'s and $k_i$'s. The $\alpha_i$'s must be chosen in a way such that the poles of the error differential equations (see proofs for the theorems) must lie in the negative left half plane. A higher value to these means more \textit{stable} control, but higher control effort. The value assigned to the $k_i$'s dictate the \textit{exponential decaying} of the error. Setting a very high value or a very low value to these constants can render the control law unstable. 
\subsection{Robustness Analysis}
In practice, there exists some unaccounted external influences from the environment like wind etc., that can affect the tracking performance of the controller. Further, there can be some uncertainties regarding the weight of the payload, and the exact value may not be known at the time of controller design. It is thus essential to analyze the effect of these external disturbances and uncertainties on the tracking performance of the proposed controller.
\subsubsection{External disturbance}
To test the robustness of the tracking controller, an external disturbance i.e., a gust of wind is generated along the positive $x-$axis, and this disturbance is applied directly to the payload. Note that the controller does not know the magnitude or the time at which this disturbance occur.
\begin{figure}[!htb]
    \centering
    \includegraphics[scale=0.5]{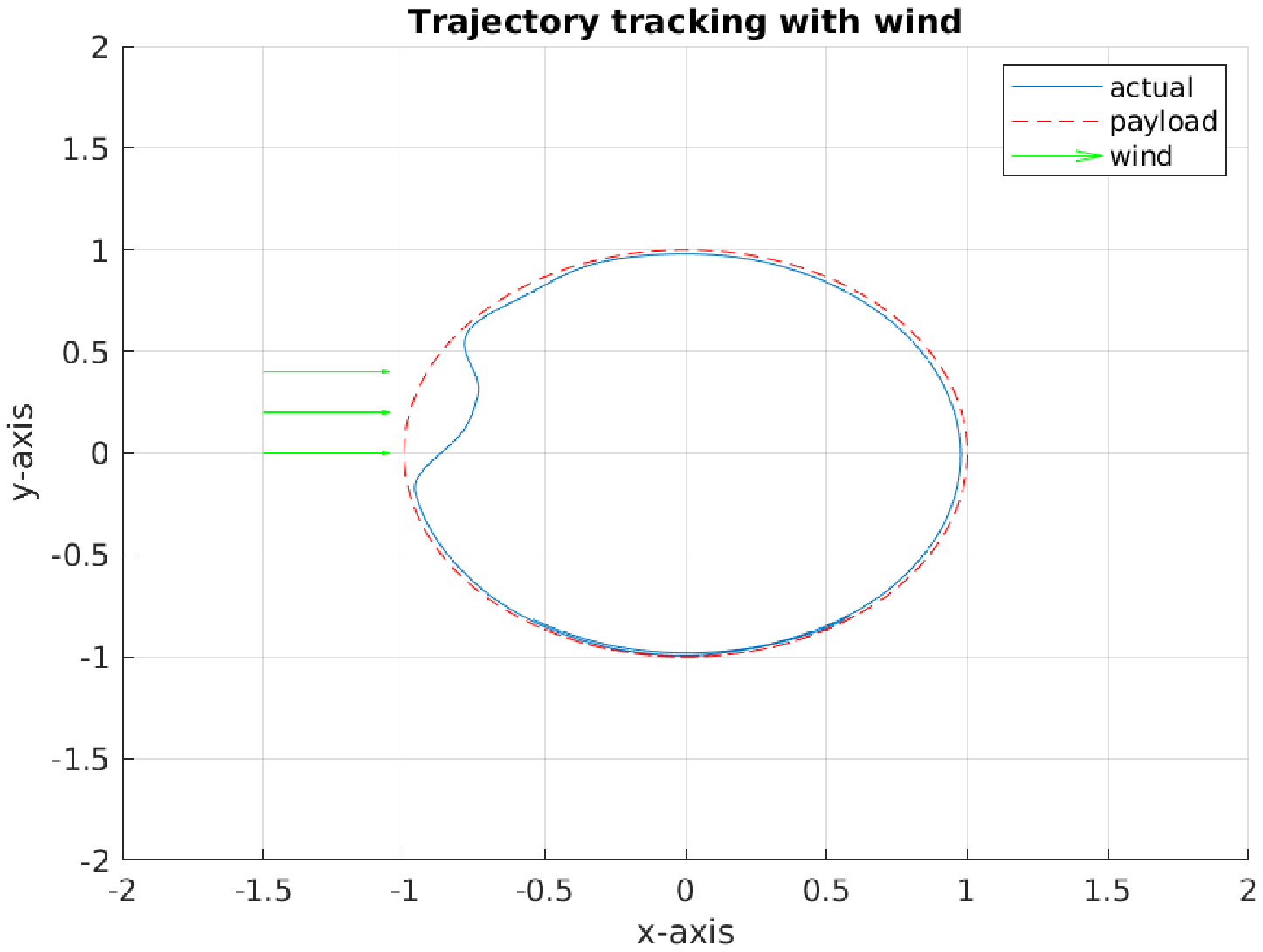}
    \caption{Caption}
    \label{fig:wind_traj_track}
\end{figure}
The tracking performance is shown in Fig. \ref{fig:wind_traj_track}, for a wind force of $5$N. The payload gets pushed along the positive x-axis, but the UAV's immediately counter this effect and get back the payload on track, to continue tracking the desired trajectory. The overall deviation is about $0.23$m from the desired trajectory.

\subsubsection{Mass Uncertainty}
In reality, there is some uncertainty between the actual payload mass and the reported payload mass. A variation of $\pm10\%$ in payload mass is considered for the actual mass of the payload i.e., the mass can be either $2.79kg$ or $3.41kg$ for the actual weight of $3.1kg$ as reported in Table \ref{tab:paramVals}. The uncertain mass values are taken during the controller design, however the plant equations have the exact mass of the payload. It is seen that the tracking performance doesn't change, however the control effort $\left\Vert \mathbb{U} \right\Vert_2$ changes by $\pm7\%$. This change mainly occurs in the thrust value of the UAVs, which is evident to compensate the weight difference of the payload. As long as the variation in mass falls under the lift capacity of the UAVs combined, the tracking performance remains the same.

\subsection{Hi-fidelity Software-in-Loop Simulation}
In order to evaluate realistic performance of the proposed control law, the high-fidelity \texttt{Gazebo} simulator is selected. The software-in-loop simulation evaluates the real-time performance of the controller and its computational efficiency. The controller is deployed at a frequency of $50$Hz in the 
simulation, and the following three stages occur:
\begin{itemize}
    \item The UAVs are spawned randomly around the payload-UAV frame. They must come in formation and join to the attachment points (above each link of the payload-frame) so that they can transport the payload from point $A$ to point $B$. Till this point, the control assumes the \textit{extended UAV dynamics} and follows Eq. (\ref{eq:ctrl_law_uav}),
    \item Once the attachment process is complete, the controller assumes the \textit{payload-UAV dynamics} and follows Eq. (\ref{eq:ctrl_law_pl_uav}). Here the UAVs ensure that the payload center of mass tracks the provided reference trajectory.
    \item Once delivered to the location, the UAVs disconnect themselves and continue doing their previously assigned work.
\end{itemize}
For this purpose, an environment is setup in \texttt{Gazebo} simulator that the payload-UAV system must navigate. The above three stages are illustrated in Fig. \ref{fig:pay_uav_stages}. In the beginning, the payload-frame is spawned at the location $[-71.5, -104.37, 0.0]^T$, and the four UAVs have to reach the locations $[-71, -103.87, 3.6]^T$, $[-71, -104.87, 3.6]^T$, $[-72, -104.87, 3.6]^T$ and $[-72, -103.87, 3.6]^T$ respectively, which are calculated based on the values of $\rho_i$'s as given in Table \ref{tab:paramVals}. Once the UAVs attach themselves to the payload frame, the payload is made to track a custom trajectory that goes around a few obstacles and lands on a helipad platform located at $[-38.073, -165.647, 0.0]^T$. Once landed, the UAVs disintegrate to various random locations around this point. The tracking performance of the UAVs (in the first stage) as well as the payload (in the second stage) are shown in Fig. \ref{fig:traj_track_UAV_1} - \ref{fig:traj_track_UAV_4} and Fig. \ref{fig:traj_track_pay} respectively. A video demonstration of the \texttt{Gazebo} simulation can be found in \cite{nish2022supp_mat} or \href{https://youtu.be/ltCwnNr9nJ0}{here}. 
\begin{figure}[!htb]
    \includegraphics[scale=0.54]{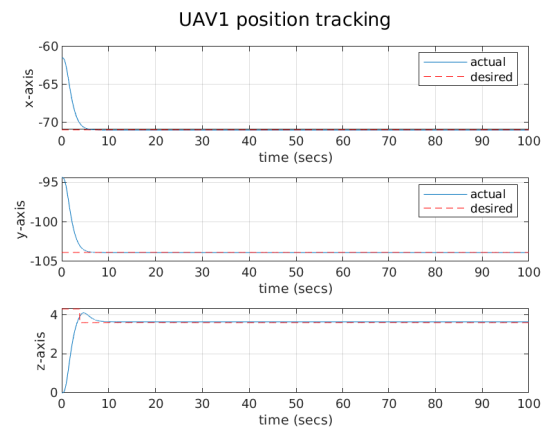}
    \caption{The reference trajectory (which is the point above link 1) for UAV 1 is shown in red dotted lines. The position of UAV 1 is shown in blue solid line.}
    \label{fig:traj_track_UAV_1}
\end{figure}
\begin{figure}[!htb]
    \includegraphics[scale=0.54]{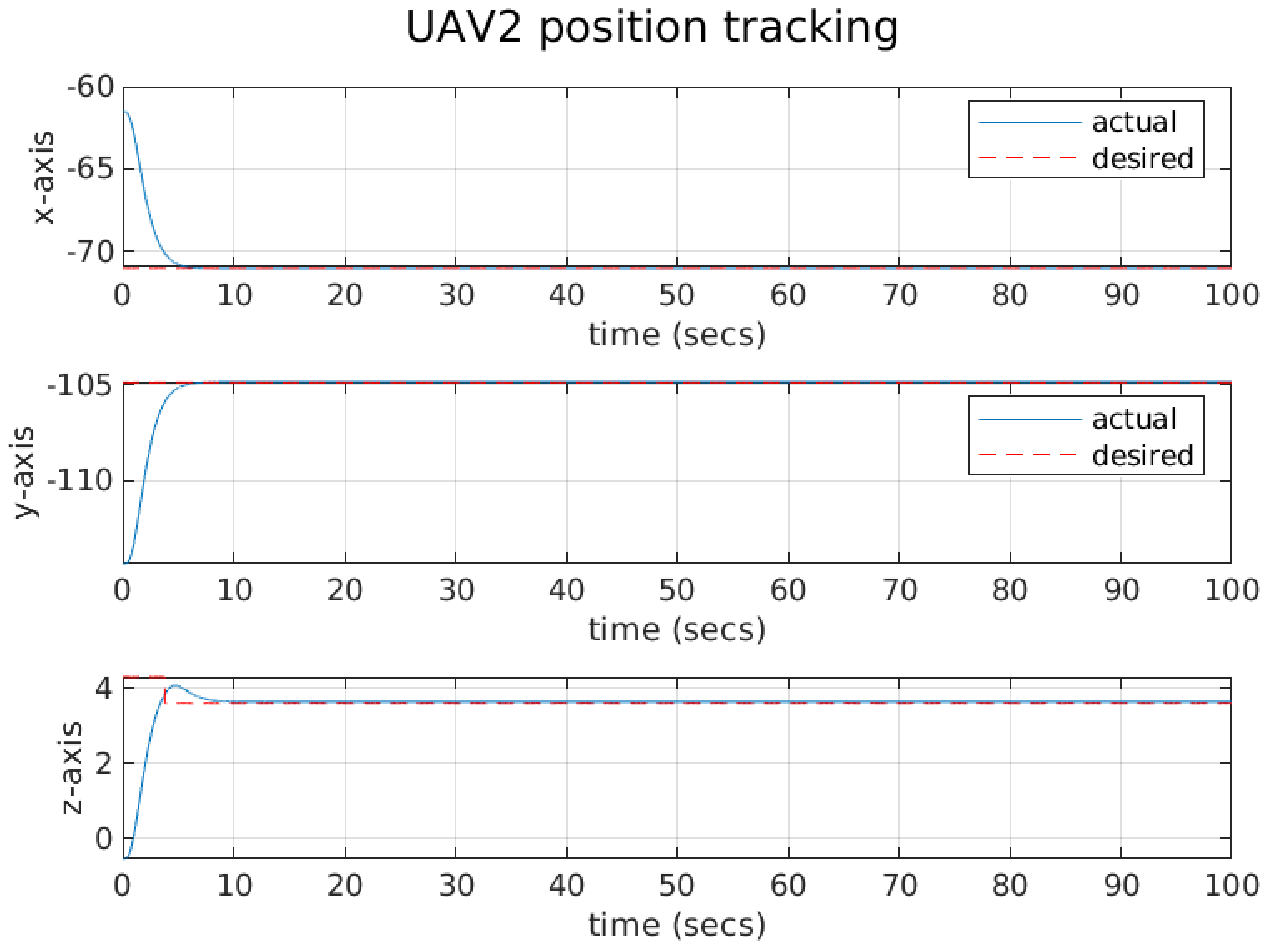}
    \caption{The reference trajectory (which is the point above link 2) for UAV 2 is shown in red dotted lines. The position of UAV 2 is shown in blue solid line.}
    \label{fig:traj_track_UAV_2}
\end{figure}
\begin{figure}[!htb]
    \includegraphics[scale=0.54]{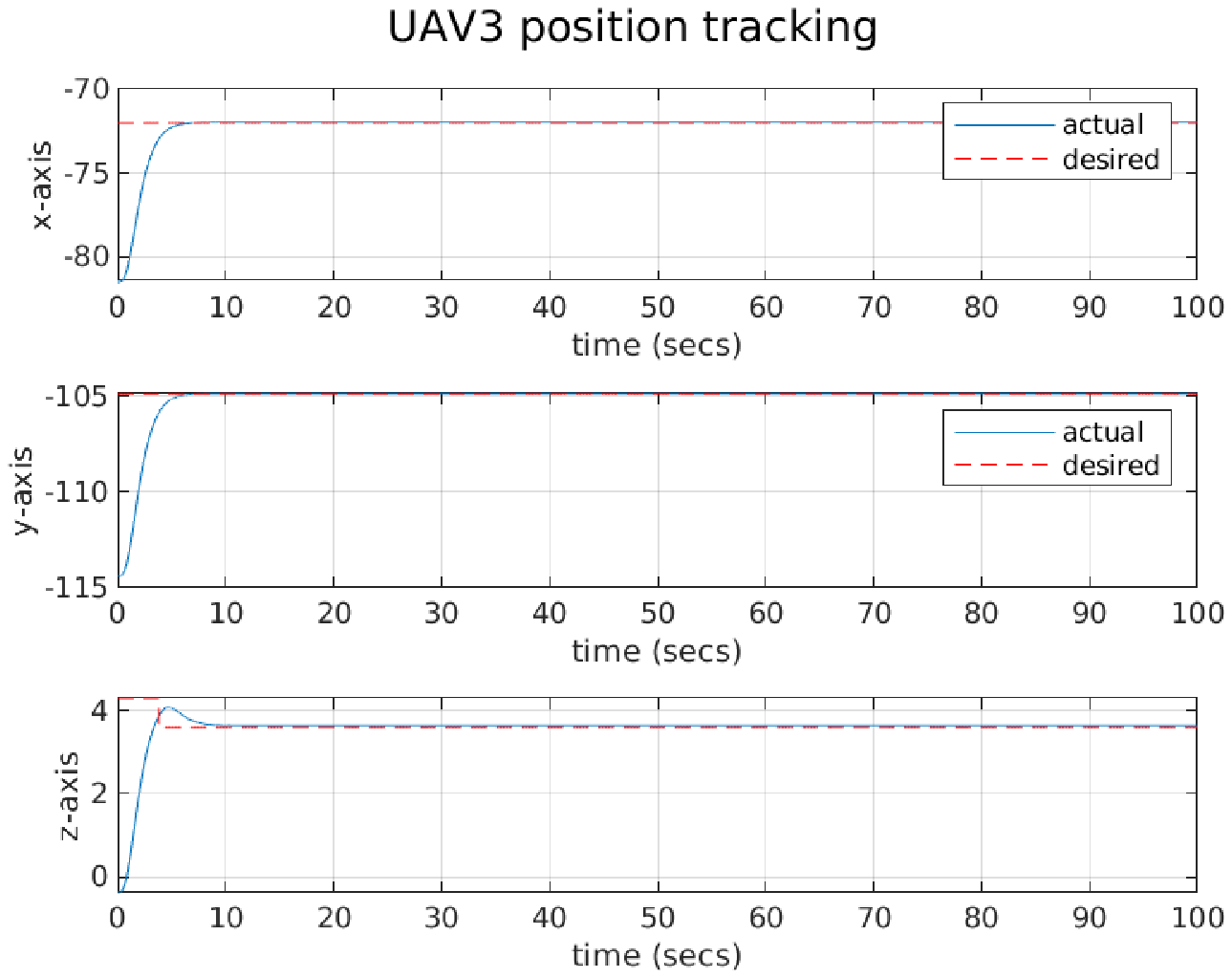}
    \caption{The reference trajectory (which is the point above link 3) for UAV 3 is shown in red dotted lines. The position of UAV 3 is shown in blue solid line.}
    \label{fig:traj_track_UAV_3}
\end{figure}
\begin{figure}[!htb]
    \includegraphics[scale=0.54]{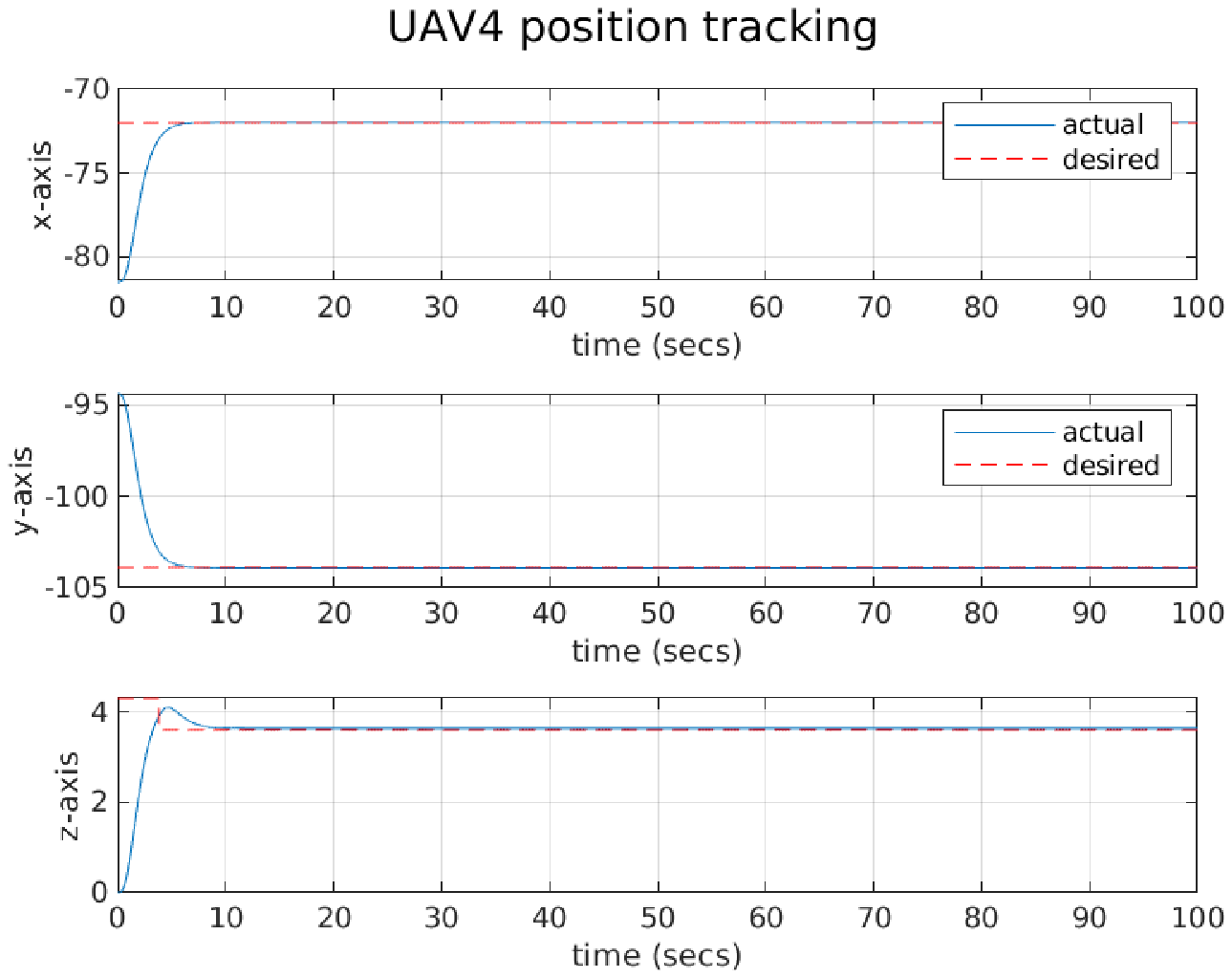}
    \caption{The reference trajectory (which is the point above link 4) for UAV 4 is shown in red dotted lines. The position of UAV 4 is shown in blue solid line.}
    \label{fig:traj_track_UAV_4}
\end{figure}
\begin{figure}[!htb]
    \centering
    \includegraphics[scale=0.54]{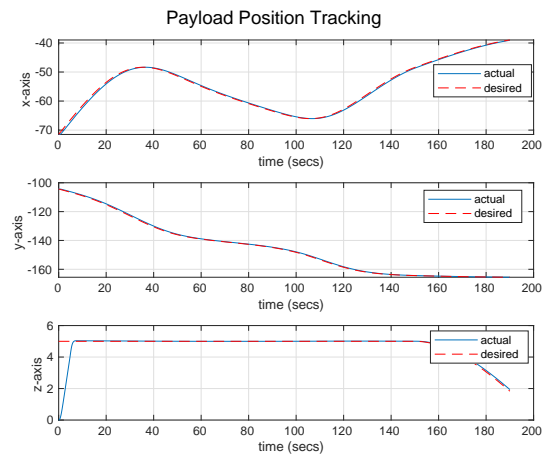}
    \caption{The reference trajectory for the payload is shown in red dotted line. The position of the payload is shown in blue solid line. The reference trajectory gradually descends to the helipad landing point.}
    \label{fig:traj_track_pay}
\end{figure}
From the UAV tracking figures, one can observe that the reference height is at $4.3$m, which is reduced to $3.6$m as soon as the UAVs come close to the attachment points' location. In other words, the UAVs are made to descend on top of the magnetic joints so that the magnetic force pulls the UAVs vertically down. If the UAVs come too close to the magnetic joints in any other direction, the UAVs can face a pulling force in some non-vertical direction which will cause it to unnecessarily roll/pitch.

\section{Conclusions}
\label{sec:future}
An exponentially stable trajectory tracking nonlinear controller is designed for a multi-UAV payload transport. The payload-UAV system consists of vertical massless links rigidly attached to the payload frame. The UAVs are connected to the links via a 2-DOF magnetic spherical joint. A novel input-output feedback linearization for combined multi-UAV and payload has been derived, and thrust vectoring control is proposed to provide stable trajectory tracking performance. The theoretical analysis clearly indicates that the proposed nonlinear control law is exponentially stable. Numerical simulation and a high-fidelity \texttt{Gazebo} simulation are conducted to validate the tracking performance of the controller and its deployment in a real-time scenario, making it computationally efficient. Robustness studies are conducted to analyse the tracking performance of the proposed controller under certain practical dubieties like an external disturbance on the payload and payload mass uncertainty. The proposed nonlinear controller is computationally efficient and can handle uncertainties coherently. Since the vertical links are rigidly attached to the payload frame, the payload can sway back and forth when the UAVs stop or move abruptly. Hence, one needs to generate a smooth and continuous desired trajectory. 
\section{Acknowledgements}
The authors would like to acknowledge the financial support from the Nokia CSR grant on Network Robotics.

\bibliographystyle{IEEEtran}
\bibliography{references}
\nocite{*}

\end{document}